\documentclass[10pt]{article} 
\usepackage[accepted]{TMLR_package/tmlr}

\usepackage{amsmath,amsfonts,bm}

\def\eqref#1{equation~\ref{#1}}

\def\1{\bm{1}}

\DeclareMathAlphabet{\mathsfit}{\encodingdefault}{\sfdefault}{m}{sl}
\SetMathAlphabet{\mathsfit}{bold}{\encodingdefault}{\sfdefault}{bx}{n}

\usepackage{amsmath}
\usepackage{amssymb}
\usepackage{amsthm}
\usepackage{amstext}
\usepackage{amscd}
\usepackage{amsfonts}
\usepackage{booktabs}
\usepackage{enumerate}
\usepackage{graphicx}
\usepackage{latexsym}
\usepackage{mathrsfs}
\usepackage{mathtools}
\usepackage{cases}
\usepackage{verbatim}
\usepackage{bm}
\usepackage{cases}  
\usepackage{empheq}
\usepackage{graphicx}
\usepackage{hhline}
\usepackage{tikz}
\usepackage[top=3cm, bottom=3cm, left=2.5cm, right=2.5cm]{geometry}
\usepackage[labelformat=empty,labelsep=none]{subcaption}
\usepackage{url}
\usepackage{comment}
\theoremstyle{plain}
\newtheorem{thm}{Theorem}
\newtheorem{cor}[thm]{Corollary}
\theoremstyle{definition}

\newcommand{\pd}{\partial}
\usepackage{algorithm}
\usepackage{algorithmic}
\usepackage{color}
\usepackage{titlesec}
\usepackage{hyperref}
\usepackage{multirow}
\usepackage{adjustbox}

\mathtoolsset{showonlyrefs=true}

\title{Spike Accumulation Forwarding for Effective Training \\ of Spiking Neural Networks}

\author{\name Ryuji Saiin\footnotemark[1]
\email ryuji.saiin@aisin-software.com \\
      \addr Tokyo Research Center, AISIN, Tokyo, Japan\\
      AISIN SOFTWARE, Aichi, Japan 
      \AAND
      \name Tomoya Shirakawa\footnotemark[1] \email kinezdayo@gmail.com \\
      \addr Graduate School of Mathematics, Nagoya University, Aichi, Japan
      \AAND
      \name Sota Yoshihara\thanks{Equal contribution.} \thanks{Corresponding author.} \email sota.yoshihara.e6@math.nagoya-u.ac.jp.\\
      \addr Graduate School of Mathematics, Nagoya University, Aichi, Japan
      \AAND
      \name Yoshihide Sawada \email yoshihide.sawada@gmail.com \\
      \addr Tokyo Research Center, AISIN, Tokyo, Japan
      \AAND
      \name Hiroyuki Kusumoto \email kusumoto-108@outlook.com \\
      \addr Graduate School of Mathematics, Nagoya University, Aichi, Japan}

\def\month{06}  
\def\year{2024} 
\def\openreview{\url{https://openreview.net/forum?id=RGQsUQDAd9}} 

\begin{document}

\maketitle

\begin{abstract}
In this article, we propose a new paradigm for training spiking neural networks (SNNs), {\it spike accumulation forwarding} ({\it SAF}). It is known that SNNs are energy-efficient but difficult to train. Consequently, many researchers have proposed various methods to solve this problem, among which online training through time (OTTT) is a method that allows inferring at each time step while suppressing the memory cost. However, to compute efficiently on GPUs, OTTT requires operations with spike trains and weighted summation of spike trains during forwarding. In addition, OTTT has shown a relationship with the Spike Representation, an alternative training method, though theoretical agreement with Spike Representation has yet to be proven. Our proposed method can solve these problems; namely, SAF can halve the number of operations during the forward process, and it can be theoretically proven that SAF is consistent with the Spike Representation and OTTT, respectively. Furthermore, we confirmed the above contents through experiments and showed that it is possible to reduce memory and training time while maintaining accuracy.
\end{abstract}

\section{Introduction}
\label{sec:introduction}

\begin{figure}[tb]
\centering
  \begin{minipage}[tb]{0.475\linewidth}
    \centering
    \includegraphics[keepaspectratio, scale=0.45]{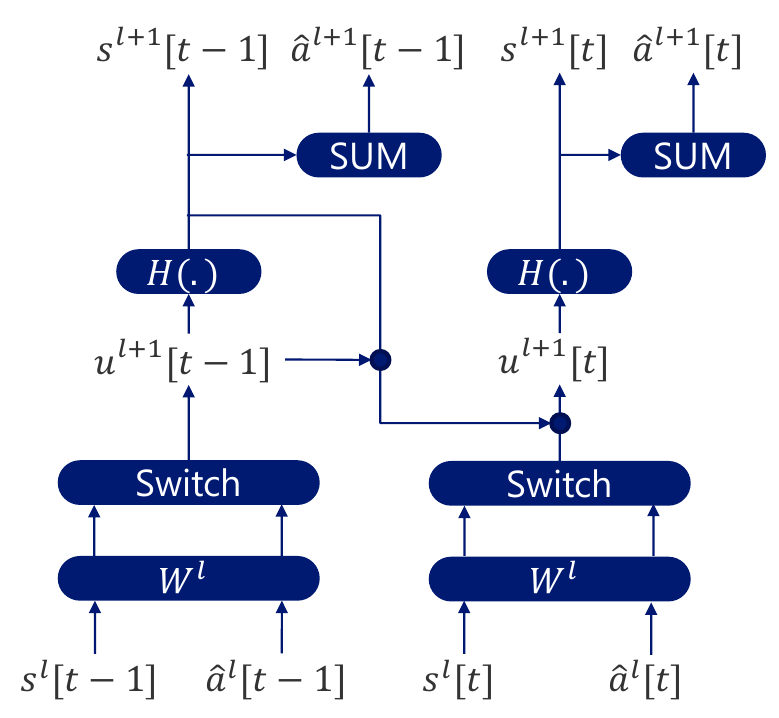}\subcaption{OTTT}
  \end{minipage}
  \begin{minipage}[tb]{0.475\linewidth}
    \centering
    \includegraphics[keepaspectratio, scale=0.45]{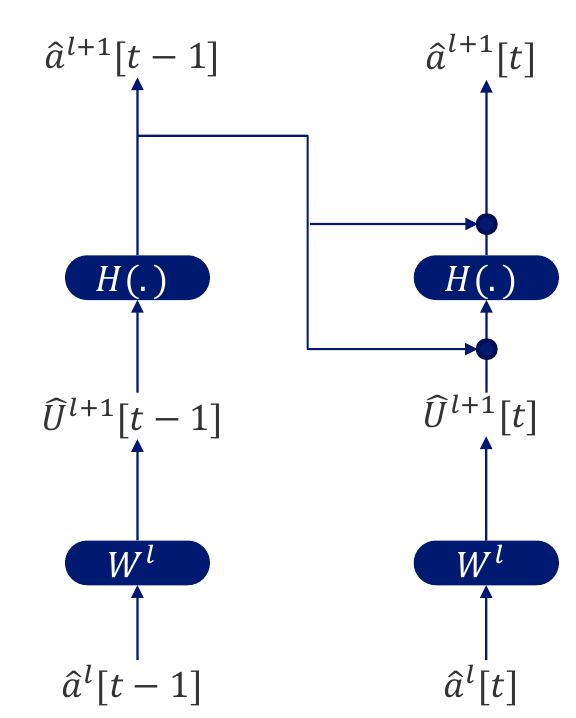} \subcaption{SAF}
  \end{minipage}
  \caption{Overview of OTTT and SAF training. OTTT requires operations with spike $\bm{s}^l[t]$ and spike accumulation $\widehat{\bm{a}}^l[t]$ during the forward process, while SAF requires operations with $\widehat{\bm{a}}^l[t]$. Also, $\bm{u}^l[t]$ represents the membrane potential, and $\widehat{\bm{U}}^{l+1}[t]$ represents the potential accumulation~(see Sec.~\ref{sec:saf}). The SUM layer computes the spike accumulation, and the Switch layer propagates $\bm{W}^l\bm{s}$ and $\bm{W}^l\widehat{\bm{a}}$ in forwarding and backwarding, respectively. Note that, unlike membrane potentials, the potential accumulation does not require the retention of past information.}
  \label{fig:overview}
\end{figure}

Due to  the carbon emission reduction problem, energy-efficient spiking neural networks~(SNNs) are attracting attention~\citep{luo2023achieving}. SNNs are known to be more bio-plausible models than artificial neural networks~(ANNs) and can replace the multiply-accumulate~(MAC) operations with additive operations. This characteristic comes from propagating the spike train~(belonging to $\{0,1\}^T$, where $T$ is the number of time steps) and is energy-efficient  on  neuromorphic chips~\citep{akopyan2015truenorth,davies2018loihi}.

Despite the usefulness of SNNs for ${\rm CO}_2$ reduction, their neurons are non-differentiable, which makes them difficult to train. Solving this problem is in the mainstream of SNN research, and back-propagation through time (BPTT) with surrogate gradient (SG)~\citep{zheng2021going,Xiao2022OnlineNetworks} is one of the main methods to achieve high performance. In particular, the recently proposed Online Training Through Time (OTTT) can train SNNs at each time step, just like our brains, and achieve high performance with few time steps~\citep{Xiao2022OnlineNetworks}. 

To enable online training, OTTT uses different information for the forward and backward processes. For forwarding, the spike train is used; for backwarding, the weighted summation of spike trains~(which we refer to as {\it spike accumulation}) is used. Therefore, efficient computation on GPUs using the Autograd of libraries such as PyTorch~\citep{paszke2019pytorch} requires operations with spike train and spike accumulation during the forward process~(see Fig.~\ref{fig:overview}). Meanwhile, OTTT has the theoretical guarantee that the gradient descent direction is similar to that of Spike Representations by the weighted firing rate coding by summing up the gradients of each time step~\citep{Xiao2021TrainingState,meng2022training}. However, these gradients are not shown to be perfectly consistent. To accurately bridge them, it is essential to develop a method that guarantees the gradient can be consistent with each of above two gradients.

In this article, we propose {\it Spike Accumulation Forwarding} ({\it SAF}) as a new paradigm for training SNNs. Unlike OTTT, SAF propagates not only backward but also forward processes by spike accumulation~(see Fig.~\ref{fig:overview}). By using this process, we can halve the number of operations during the forward process. In addition, because SAF does not require retaining the information of membrane potentials as in \citet{zhou2021temporal}, we can reduce memory usage during training compared to OTTT. Furthermore, this propagation strategy is only executed during training, and during inference, we can replace the propagation with the spike train without approximation error. We demonstrate this by proving that the neurons for spike accumulation are identical to the Leaky-Integrate-and-Fire~(LIF)~\citep{stein1965theoretical} neuron, which is a generalization of the Integrate-and-Fire~(IF)~\citep{lapique1907researches} neuron, which are commonly used as SNN neurons. This result indicates that the SNN composed of LIF neurons can achieve the same accuracy using the trained parameters of SAF (i.e., SAF is capable of inference by the SNN composed of LIF neurons). Furthermore, we prove that the gradient of the SAF is consistent with that of the OTTT, which trains at each time step, and that by summing up the gradients at each time step, SAF is also consistent with the gradient of the Spike Representation.
This shows that SAF can accurately bridge the gap between Spike Representation and OTTT.

{\bf Main Contributions}
\begin{itemize}
\item[(A)] We propose SAF, which trains SNNs by only spike accumulation, halving the number of operations in the forward process, reducing the memory cost, and enabling inference on SNNs composed of LIF neurons.  
\item[(B)] We prove theoretically that the neurons for spike accumulation are absolutely identical to the LIF neuron.  
\item[(C)] Our study also shows that the gradient of SAF is theoretically consistent with the gradient of the Spike Representation and also with that of OTTT, which trains each time step. 
\item[(D)] We consider the situation that SNNs have a feedforward or feedback connection like brain and discuss the equivalence of SAF with OTTT and with Spike Representation. 
\item[(E)] Brief experiments confirmed that for training at each time step, the training results were in close agreement with OTTT while reducing the training cost. 
\end{itemize}

\section{Related Work}
Regarding SNN training, there are two research directions: conversion from ANN to SNN and direct training. The conversion approach reuses the parameters of the ANN while converting the activation function for the spiking function~\citep{diehl2015fast,deng2020optimal,han2020rmp}. This approach can be employed by a wide range of many trained deep-learning models, and there are use cases for tasks other than recognition~\citep{kim2020spiking,qu2023spiking}. However, because the accuracy tends to be proportional to the number of time steps and although several improvement methods have been proposed~\citep{chowdhury2021one,wu2021tandem}, high-precision inference is still difficult for few time steps. Meanwhile, direct training does not use the parameters of the trained ANNs. Thus, the non-differentiable SNNs are trained by some approximation techniques. One of the most significant techniques is to utilize the surrogate gradient~(SG). SGs enables backpropagation by approximating the gradient of non-differentiable activation functions, and various types of SGs (e.g., rectangular or derivatives of sigmoid functions) have been proposed~\citep{shrestha2018slayer,wu2018spatio,ijcai2023p335,Suetake2022S3NN:Networks}. Other recently proposed methods 
include those based on the clamp function~\citep{meng2022training} or implicit differentiation on the equilibrium state~\citep{Xiao2021TrainingState}. These employ Spike representation, which propagates information such as firing rates, not spike trains, and have the advantage of being able to train SNNs like ANNs~\citep{thiele2019spikegrad,zhou2021temporal}. However, these methods assume the time step $T \rightarrow \infty$, then $T$ must be sufficiently large to achieve high accuracy. In addition, these are only differences in how to approximate; the basic approach is similar to SGs. Although there are bio-inspired training methods, such as Hebbian learning rule~\citep{hebb2005organization,fremaux2016neuromodulated} and spike timing dependent plasticity~(STBP)~\citep{bi1998synaptic,bengio2015stdp}. In particular, learning rules based on eligibility traces, such as SuperSpike~\citep{zenke2018superspike}, a method for improving SpikeProp~\citep{bohte2002error}, are associated with three-factor plasticity rules~\citep{neftci2019surrogate}. OTTT, the main focus of our study, has also been mentioned as being associated with the three-factor plasticity rule. However, training multi-layered SNNs with SuperSpike is challenging, and it is necessary to introduce local errors in each layer~\citep{kaiser2020synaptic}. Nonetheless, in \citet{kaiser2020synaptic}, fewer than ten layers were trained, and training deeper SNNs has still been challenging. On the other hand, OTTT can train deeper networks such as VGG, which is useful for many applications.

In the following, we discuss the OTTT~\citep{Xiao2022OnlineNetworks} most relevant to our study. Because OTTT is a variant of BPTT with SG, it allows for low-latency training. In addition, it is sufficient for OTTT to maintain the computational graph only for the current time step during training, different from the standard BPTT with SG. Thus, training can be performed with constant memory usage even as time steps increase. However, OTTT requires additional information for propagating the spike accumulation as well as the spike train for the forward process, which can increase training time. In addition, because OTTT is based on the LIF neuron, it must retain the membrane potential, which can increase memory usage. Furthermore, it is important to note that while OTTT and Spike Representation have similar gradient directions, their gradients do not necessarily match exactly (i.e., the inner product between their gradients is positive). 

\section{Preliminaries}
\subsection{Typical Neuron Model}
In this subsection, we explain the LIF neuron, which is widely used in SNNs. The LIF neuron is a neuron model that considers the leakage of the membrane potential, and its discrete computational form is as follows:
\begin{equation}
\begin{aligned}
\begin{cases}
\bm{u}^{l+1}[t] = \lambda(\bm{u}^{l+1}[t-1] - V_{\rm th} \, \bm{s}^{l+1}[t-1]) + \bm{W}^l \bm{s}^{l}[t] + \bm{b}^{l+1},\\
\bm{s}^{l+1}[t] = H(\bm{u}^{l+1}[t] - V_{\rm th}),
\end{cases}
\end{aligned}
\label{define-snn}\end{equation}
where $\bm{s}^l[t]$, $\bm{u}^l[t]$, $\bm{W}^l$ and $\bm{b}^l$ are the spike train, membrane potential, weight and bias of $l$-th layer, respectively. $\lambda \le 1$ is the leaky term, and $\lambda$ is set to 1 if we use the IF neuron, which is a special case of the LIF neuron. Also, $H$ is the element-wise Heaviside step function, that is, $H=1$ when the membrane potential $\bm{u}[t]$ exceeds the threshold $V_{\rm th}$. From this relation, the membrane potentials $\bm{u}^l[t]$ are computed sequentially and must retain the previous membrane potential $\bm{u}^l[t-1]$. This is the same in the case of OTTT, which uses the LIF neuron as described below.

\subsection{Training methods for SNNs}\label{Training_methods_for_SNNs}
This subsection introduces two training methods that are closely related to our method: Spike Representation and OTTT. 

\subsubsection*{Spike Representation}
Spike Representation is a method of training SNNs by propagating information differently to the spike trains~\citep{Xiao2021TrainingState,meng2022training}. In this article, we consider the weighted firing rate $\bm{a}[t] = \sum_{\tau = 0}^{t}\lambda^{t-\tau}\bm{s}[\tau] / \sum_{\tau = 0}^{t}\lambda^{t-\tau}$ as in \citet{Xiao2021TrainingState,meng2022training,Xiao2022OnlineNetworks}. Likewise, we define the weighted average input $\bm{m}[t] =  \sum_{\tau = 0}^{t}\lambda^{t-\tau}\bm{x}[\tau] / \sum_{\tau = 0}^{t}\lambda^{t-\tau}$ , where $x$ is the value of the input data. Then, given a convergent sequence $\bm{m}[t]\rightarrow \bm{m}^* \ (t \rightarrow \infty)$, it is known that $ \bm{a}[t] \rightarrow$ $\sigma \left(\bm{m}^*\right / V_{\rm th}) \ (t \rightarrow \infty)$  holds \citep{Xiao2021TrainingState}, where $\sigma$ is the element-wise clamp function $\sigma(x) = \min(\max(0,x),1)$. Using this convergence and under the assumption that the time step $T$ is sufficiently large, the weighted firing rate in the ($l+1$)-th layer is approximated as $\bm{a}^{l+1}[T] \approx \sigma(\bm{W}^l\bm{a}^l[T]+\bm{b}^{l+1} / V_{\rm th})$. We consider the loss $L$ as $ L[t] = \mathcal{L}(\bm{a}[T], y )$ (where  $\mathcal{L}$ is a convex function like cross-entropy and $y$ is the label). Then the gradient of $L$ with respect to $\bm{W}^l$ is computed as follows:
\begin{equation}
\left(\frac{\pd L}{\pd \bm{W}^l}\right)_{\rm SR} = \frac{\pd L}{\pd \bm{a}^{N}[T]}\left(\prod_{i=N-1}^{l+1}\frac{\pd \bm{a}^{i+1}[T]}{\partial \bm{a}^{i}[T]}\right)\frac{\pd\bm{a}^{l+1}[T]}{\pd\bm{W}^l},
\label{grad-SR}\end{equation}
where $N$ represents the number of layers. 

Spike Representation can include a feedforward or feedback connections, which have been frequently used in recent years~\citep{Xiao2021TrainingState,Xiao2022OnlineNetworks,Xiao2023SPIDE:Networks}, where notice that, feedforward and feedback connections do not refer to the weights between adjacent layers but to additional weight matrices connecting any layers $l$ and $l'$. When there is a feedforward connection, the gradient \eqref{grad-SR} holds. However, in the case of feedback connection, it does not hold because we need to calculate implicit differentiation. See Appendix \ref{subsec:Proof of Theorem3} for a detailed explanation.

\subsubsection*{Online Training Through Time}
OTTT~\citep{Xiao2022OnlineNetworks} is a training method based on BPTT with the surrogate gradient (SG). BPTT with SG enables low-latency training; however, during training, it requires the computational graph to be maintained at each time step, resulting in substantial memory usage  when a large number of time steps are involved. OTTT solves this problem and allows for training with minimal memory consumption.

In OTTT, for the forward process, (weighted) spike accumulation $\widehat{\bm{a}}[t] = \sum_{\tau = 0}^{t}\lambda^{t-\tau}\bm{s}[\tau]$ is propagated in addition to spike trains $\bm{s}[t]$ computed by the LIF neuron.
Defining the loss at each time step as $L[t] = \mathcal{L}(\bm{s}[t],y)/T$, OTTT computes the gradient at time $t$ as follows:
\begin{equation}
\left(\frac{\pd L[t]}{\pd \bm{W}^l} \right)_{\rm OT}= \widehat{\bm{a}}^l[t]\, \frac{\pd L[t]}{\pd \bm{s}^{N}[t]}\left(\prod_{i=N-1}^{l+1}\frac{\pd \bm{s}^{i+1}[t]}{\partial \bm{s}^{i}[t]}\right)\frac{\pd \bm{s}^{l+1}[t]}{\pd \bm{u}^{l+1}[t]}.
\label{grad-OTTT}\end{equation}
Note that since this equation uses the accumulated information of spikes up to $\widehat{a}^l[t]$ at the current time $t$, OTTT implicitly trains using information up to $t$, not just the current time. This corresponds to the fact that membrane potentials accumulate information from the past. Note also that the term $\pd \bm{s} / \pd \bm{u}$ is non-differentiable at $u = V_{\rm th}$; thus, we approximate it with the SG; for example, since $\pd \bm{s}^{i+1}[t] / \pd \bm{s}^i[t]$ can decomposed into $(\pd \bm{s}^{i+1}[t] / \pd  \bm{u}^{i+1}[t] )(\bm{u}^{i+1}[t] / \pd \bm{s}^i[t])$, we also replace $\pd \bm{s}^{i+1}[t] / \pd  \bm{u}^{i+1}[t]$ with the SG. Like the case of Spike Representation, OTTT can  include a feedforward or feedback connection, and the gradients are almost the same as \eqref{grad-OTTT}.

\citet{Xiao2022OnlineNetworks} proposed two types of training approaches: OTTT$_{\rm O}$, where parameters are updated at each time step using $\pd L[t] / \pd \bm{W}^l$, and OTTT$_{\rm A}$, where parameters are updated collectively by summing $\pd L[t] / \pd \bm{W}^l$ up to $T$. 
 
In particular, they proved that the gradient descent directions in OTTT$_{\rm A}$ and Spike Representation are similar, i.e., the inner product between their gradients is positive.

\section{Spike Accumulation Forwarding}
\label{sec:saf}
In this section, we introduce our proposed method, SAF, which only propagates (weighted) spike accumulation $\widehat{\bm{a}}[t]$. We first explain the forward and backward processes. Then, we prove that SAF can be consistent with OTTT and Spike Representation. We also show that the feedback connection can be added to SAF, as in~\citet{Xiao2022OnlineNetworks}, and furthermore, feedforward connections can also be incorporated into the SAF. For summaries of the main formulas, see Appendix~\ref{list_of_main_formulas}.

\subsection{Details of SAF}
\subsubsection*{Forward process}
As mentioned earlier, for the forward processes of conventional SNNs, the spike trains $\bm{s}[t]$ are propagated. In other words, the firing state of the spike for each neuron at each time step is retained. In SAF, as in OTTT and SR, instead of the spike trains, it propagates (weighted) spike accumulation $\widehat{\bm{a}}[t] = \sum_{\tau = 0}^{t}\lambda^{t-\tau}\bm{s}[\tau]$, meaning that it retains the (weighted) count of the fired spikes up to the current time for each neuron.
Additionally, although in conventional SNNs, the spike firing is determined with the membrane potential $\bm{u}[t]$, in SAF, it is determined with (weighted) potential accumulation $\widehat{\bm{U}}[t]$ defined by $ \widehat{\bm{U}}^{l+1}[t] = \lambda\widehat{\bm{U}}^{l+1}[t-1] + \bm{W}^l(\widehat{\bm{a}}^l[t]-\lambda\widehat{\bm{a}}^l[t-1]) + \bm{b}^{l+1}$, which corresponds to the membrane potential in the relation~\eqref{define-snn}. 
With these considerations, SAF is updated as follows:
\begin{equation}
\begin{aligned}
\begin{cases}
\widehat{\bm{U}}^{l+1}[t] = \bm{W}^l \widehat{\bm{a}}^l[t] + \bm{b}^{l+1}\sum_{\tau = 0}^{t-1}\lambda^{t-\tau} + \lambda^{t}\widehat{\bm{U}}^{l+1}[0],\vspace{2pt}\\
\widehat{\bm{a}}^{l+1}[t] = \lambda\widehat{\bm{a}}^{l+1}[t-1] + H(\widehat{\bm{U}}^{l+1}[t] -  V_{\rm th}(\lambda\widehat{\bm{a}}^{l+1}[t-1] +1)),
\end{cases}
\end{aligned}
\label{SAF-Def}\end{equation}
where $ \widehat{\bm{U}}^{l+1}[0]$ is the initial value for potential accumulation, and here, we assume it to be the initial membrane potential $ \bm{u}^{l+1}[0]$.
Here, the membrane potential $\bm{u}^{l+1}[t]$ and spike trains $\bm{s}^{l+1}[t]$ in the LIF model can be expressed by $\widehat{\bm{U}}^{l+1}[t]$ and $\widehat{\bm{a}}^{l+1}[t-1]$ as follows:
\begin{equation}
\begin{aligned}
\begin{cases}
\bm{u}^{l+1}[t] = \widehat{\bm{U}}^{l+1}[t] - V_{\rm th} \,\lambda\widehat{\bm{a}}^{l+1}[t-1], \\
\bm{s}^{l+1}[t] = H(\widehat{\bm{U}}^{l+1}[t] -  V_{\rm th}(\lambda\widehat{\bm{a}}^{l+1}[t-1] +1)) .
\end{cases}
\end{aligned}
\label{eq:SAF-LIF}\end{equation}
For derivations of \eqref{SAF-Def} and \eqref{eq:SAF-LIF}, refer to Appendix~\ref{subsection_derivation_4_5}. 
Note that, as shown in \eqref{SAF-Def}, SAF does not need to retain the past potential accumulation $\widehat{\bm{U}}^{l+1}[t-1]$. Meanwhile, the various SNNs, including OTTT, require the LIF neurons used for training to retain the previous membrane potentials $\bm{u}^{l+1}[t-1]$, as described above. Therefore, SAF can reduce the memory usage for the forward process compared to OTTT. 

As a result, it becomes possible to compute $\bm{u}^{l+1}[t]$ and $\bm{s}^{l+1}[t]$ during the process of obtaining $\widehat{\bm{U}}^{l+1}[t]$ and $\widehat{\bm{a}}^{l+1}[t]$. Because it is possible to compute $\widehat{\bm{U}}^{l+1}[t]$ and $\widehat{\bm{a}}^{l+1}[t]$ from $\bm{u}^{l+1}[t]$ and $\bm{s}^{l+1}[\tau]$ ($\tau = 1,\ldots,t$), the forward processes of SAF and SNN composed of LIF neurons are mutually convertible.  Additionally, because the IF neuron is a special case of an LIF neuron (i.e., $\lambda=1$), the forward processes of SAF, when $\lambda=1$, and SNN composed of IF neurons are mutually convertible. Furthermore, in OTTT, both $\bm{s}[t]$ and $\widehat{\bm{a}}[t]$ need to be propagated during the forward process for efficient GPU computation, whereas in SAF, only $\widehat{\bm{a}}[t]$ needs to be propagated~(see Fig.~\ref{fig:overview}).
Therefore, SAF can reduce the computation time during training. 

\subsubsection*{Backward process}
As with OTTT, SAF can be trained in two different ways. The first method updates the parameters by computing the gradient at each time step. We call this {\it SAF-E}. Let $L_E[t]=\mathcal{L}(\bm{s}^N[t],\bm{y})/T$ be the loss function.
Assuming that $L_E[t]$ depends only on $\widehat{\bm{a}}^l[t]$ and $\widehat{\bm{U}}^l[t]$, i.e., not on anything up to $t-1$, we calculate the derivative based on the definition of forward propagation as
\begin{align}
\frac{\partial L_E[t]}{\partial \bm{W}^l}&=\widehat{\bm{a}}^l[t]\,\frac{\partial L_E[t]}{\partial \widehat{\bm{a}}^N[t]}\left(\prod_{i=N-1}^{l+1}\frac{\partial \widehat{\bm{a}}^{i+1}[t]}{\partial \widehat{\bm{a}}^{i}[t]}\right)\frac{\partial \widehat{\bm{a}}^{l+1}[t]}{\partial \widehat{\bm{U}}^{l+1}[t]} .\label{grad-SAF-o2}
\end{align}
Note that $\partial \widehat{\bm{a}} / \partial \widehat{\bm{U}}$ is non-differentiable; we approximate it with SG. Detailed calculations are given in Appendix~\ref{subsec:6_7}. Here, we set
\begin{align}
\bm{g}_{\widehat{\bm{U}}}^{l+1}[t]
&=\frac{\partial L_E[t]}{\partial \widehat{\bm{a}}^N[t]}\left(\prod_{i=N-1}^{l+1}\frac{\partial \widehat{\bm{a}}^{i+1}[t]}{\partial \widehat{\bm{a}}^{i}[t]}\right)\frac{\partial \widehat{\bm{a}}^{l+1}[t]}{\partial \widehat{\bm{U}}^{l+1}[t]}.
\end{align}
Then, it holds that
\begin{align}
\frac{\partial L_E[t]}{\partial \bm{W}^l}&=\widehat{\bm{a}}^l[t]\,\bm{g}_{\widehat{\bm{U}}}^{l+1}[t]. \label{gr-SAF-E}
\end{align}
The second method calculates the gradient only at the final time step and updates the parameters. We call this {\it SAF-F}. Let $L_F=\mathcal{L}(\sum_{t =0}^{T}\lambda^{T-t}\bm{s}^N[t] / \sum_{t =0}^{T}\lambda^{T-t},\bm{y})$ be a loss function. As with SAF-E, suppose that $L_F$ depends only on $\widehat{\bm{a}}^l[T]$ and $\widehat{\bm{U}}^l[T]$. Simply replacing $t$ with $T$ and $L_E$ with $L_F$ in the above calculation, we obtain
\begin{align}
\frac{\partial L_F}{\partial \bm{W}^l}&=\widehat{\bm{a}}^l[T]\,\bm{g}_{\widehat{\bm{U}}}^{l+1}[T]. \label{gr-SAF-F}
\end{align}

\subsection{Equivalence with OTTT$_{\rm O}$ and Spike Representation } \label{subsec_equivalence}

In this subsection, we show that SAF-E is equivalent to OTTT$_{\rm O}$ and SAF-F is equivalent to Spike Representation, i.e., the forward and backward processes are consistent, respectively. This means that we can train SNNs by SAF and infer by LIF neurons.
\subsubsection*{Equivalence with OTTT$_{\rm O}$}

We will transform the gradient of SAF-E to be consistent with that of OTTT$_{\rm O}$ when the loss function is $L_E[t]$.
Because $L_E[t]$ does not include any argument before $t$, we obtain 
\begin{align}
\frac{\partial L_E[t]}{\partial \widehat{\bm{a}}^N[t]}=\frac{\partial L_E[t]}{\partial \bm{s}^N[t]}.\label{equal-1}
\end{align}
The following two equations hold from the  forward processes of SAF and OTTT: 
\begin{align}
\frac{\partial \widehat{\bm{a}}^{i+1}[t]}{\partial \widehat{\bm{a}}^{i}[t]}=\frac{\partial \bm{s}^{i+1}[t]}{\partial \bm{s}^{i}[t]}, \quad
\frac{\partial \widehat{\bm{a}}^{l+1}[t]}{\partial \widehat{\bm{U}}^{l+1}[t]}=\frac{\partial \bm{s}^{l+1}[t]}{\partial \bm{u}^{l+1}[t]}. \label{equal-2}
\end{align}
By substituting \eqref{equal-1}, \eqref{equal-2} for $\bm{g}_{\widehat{\bm{U}}}^{l+1}$ in \eqref{gr-SAF-E}, we have
\begin{align}
\bm{g}_{\widehat{\bm{U}}}^{l+1}[t] &=\frac{\partial L_E[t]}{\partial \bm{s}^N[t]}\left(\prod_{i=N-1}^{l+1}\frac{\partial \bm{s}^{i+1}[t]}{\partial \bm{s}^{i}[t]}\right)\frac{\partial \bm{s}^{l+1}[t]}{\partial \bm{u}^{l+1}[t]}.
\end{align}

Hence, the following theorem holds from \eqref{grad-OTTT}.
\begin{thm}
The backward processes of SAF-E and OTTT$_{\rm O}$ are identical, that is, $\dfrac{\partial L_E[t]}{\partial \bm{W}^l}=\left(\dfrac{\partial L_E[t]}{\partial \bm{W}^l}\right)_{\rm OT}$. 
\label{thm1}\end{thm}
A detailed proof is given in Appendix~\ref{subsec_thm1}. Because we have already confirmed that the forward process is consistent, SAF-E and OTTT$_{\rm O}$ are equivalent.
\subsubsection*{Equivalence with Spike Representation}
Now, we show that SAF-F is equivalent to Spike Representation. Setting the loss function of Spike Representation as $L_F$, from the expression \eqref{grad-SR}, we have
\begin{align}
\left(\frac{\pd L_F}{\pd \bm{W}^l}\right)_{\rm SR}=\frac{\widehat{\bm{a}}^l[T]}{V_{\rm th}}\left(\frac{\partial L}{\partial \widehat{\bm{a}}^N[T]}\left(\prod_{i=N-1}^{l+1}\frac{\partial \widehat{\bm{a}}^{i+1}[T]}{\partial \widehat{\bm{a}}^{i}[T]}\right)\odot\bm{d}^{l+1}[T]^{\top}\right),
\end{align}
where $\bm{d}^{l+1}[T]=\sigma'\left((\bm{W}^l\widehat{\bm{a}}^l[T] / \Lambda +\bm{b}^{l+1}) / V_{\rm th} \right)$,
$\Lambda = \sum_{\tau=0}^T\lambda^{T-\tau}$, and $\odot$ is the element-wise product.
Now we assume that
\begin{align}
\frac{\partial \bm{s}^{l+1}[T]}{\partial \bm{u}^{l+1}[T]}=\operatorname{diag}(\bm{d}^{l+1}[T]),
\end{align}
for any $l=0,\ldots,N-1$, where $\operatorname{diag}(\bm{d}^{l+1}[T])$ is a diagonal matrix constructed from $\bm{d}^{l+1}[T]$. The reason why this assumption is valid discussed in~\citet{Xiao2022OnlineNetworks}.
Then, we obtain
\begin{align}
\left(\frac{\partial L_F}{\partial \bm{W}^l}\right)_{\rm SR}&=\frac{1}{V_{\rm th}}\widehat{\bm{a}}^l[T]\,\bm{g}_{\widehat{\bm{U}}}^{l+1}[T].
\end{align}
Hence, the following theorem holds.
\begin{thm} \label{thm2}
Suppose that $\bm{m}[t]$ converges when $t\rightarrow\infty$. Then, for sufficiently large $T$, the backward processes of SAF-F and Spike Representation are identical up to a scale factor, that is,
$\dfrac{\partial L_F}{\partial \bm{W}^l}=V_{\rm th}\left(\dfrac{\partial L_F}{\partial \bm{W}^l}\right)_{\rm SR}$. 
\end{thm}
See Appendix~\ref{Proof of Theorem2} for the complete proof. Because we have already confirmed that the forward process is consistent, SAF-F and Spike Representation are equivalent.

\subsection{Feedforward and Feedback Connection}
In the brain, there are not only layer stacking, as in normal SNNs, but also feedforward and feedback connections~\citep{semedo2022feedforward}. Therefore, it is important to consider SAF-E/F including these connections. In this subsection, we discuss the relationship between ${\rm OTTT_{O}}$, SR, and SAF-E/F in the context of both feedforward and feedback connections.
\subsubsection*{Feedforward Connection}
To begin, we consider the SNN composed of LIF neurons.
The forward process of the ($q+1$)-th layer of the SNN with a feedforward connection from the $p$-th layer to the ($q+1$)-th layer (where $q \geq p$) with weight $\bm{W}_f$ is as follows:
\begin{equation}
\begin{aligned}
\begin{cases}
\bm{u}^{q+1}[t] = \lambda(\bm{u}^{q+1}[t-1] - V_{\rm th}\,\bm{s}^{q+1}[t-1]) + \bm{W}^q \bm{s}^{q}[t] + \bm{b}^{q+1}+\bm{W}_f\bm{s}^p[t],\\
\bm{s}^{q+1}[t] = H(\bm{u}^{q+1}[t]- V_{\rm th}).
\end{cases}
\end{aligned}\label{define-snn-F}
\end{equation}
Note that the layers other than the ($q+1$)-th layer are the same as in~\eqref{define-snn}. Meanwhile, the forward process of the ($q+1$)-th layer of the SAF is as follows: 
\begin{equation}
\begin{aligned}
\begin{cases}
\widehat{\bm{U}}^{q+1}[t]= \bm{W}^q \widehat{\bm{a}}^q[t] +\bm{b}^{q+1}(\sum_{\tau=0}^{t-1}\lambda^{\tau}) + \lambda^{t}\widehat{\bm{U}}^{q+1}[0]+\bm{W}_f \, \widehat{\bm{a}}^p[t],\\
\widehat{\bm{a}}^{q+1}[t] = \lambda\widehat{\bm{a}}^{q+1}[t-1]+ H(\widehat{\bm{U}}^{q+1} -  V_{\rm th}(\lambda\widehat{\bm{a}}^{q+1}[t-1]+1)).
\end{cases}
\end{aligned}\label{SAF-Def-F}
\end{equation}
The layers other than the ($q+1$)-th layer are the same as in~\eqref{SAF-Def}. The forward processes of SAF and LIF with feedforward connection are mutually convertible.\\
Regarding the backward processes, the gradients for parameters other than $\bm{W}_f$ are the same as when there is no feedforward connection. The derivative with respect to $\bm{W}_f$ is calculated as 
\begin{align}
\frac{\partial L_E[t]}{\partial \bm{W}_f}&=\widehat{\bm{a}}^p[t]\,\frac{\partial L_E[t]}{\partial \widehat{\bm{a}}^N[t]}\left(\prod_{i=N-1}^{q+1}\frac{\partial \widehat{\bm{a}}^{i+1}[t]}{\partial \widehat{\bm{a}}^{i}[t]}\right)\frac{\partial \widehat{\bm{a}}^{q+1}[t]}{\partial \widehat{\bm{U}}^{q+1}[t]}.
\end{align}
Therefore, $\partial L_E[t] / \partial \bm{W}_{f}=\widehat{\bm{a}}^p[t]\,\bm{g}_{\widehat{\bm{U}}}^{q+1}[t]$ for SAF-E and $\partial L_F / \partial \bm{W}_{f}=\widehat{\bm{a}}^p[T]\,\bm{g}_{\widehat{\bm{U}}}^{q+1}[T]$ for SAF-F.
\subsubsection*{Feedback Connection}
 The forward process of the ($q+1$)-th layer of the SNN with a feedback connection from the $p$-th layer to the ($q+1$)-th layer (where $q<p$) with weight $\bm{W}_b$ is as follows:
\begin{equation}
\begin{aligned}
\begin{cases}
\bm{u}^{q+1}[t] = \lambda(\bm{u}^{q+1}[t-1] - V_{\rm th} \,\bm{s}^{q+1}[t-1]) + \bm{W}^q \bm{s}^{q}[t] + \bm{b}^{q+1}+\bm{W}_b \,\bm{s}^p[t-1],\\
\bm{s}^{q+1}[t] = H(\bm{u}^{q+1}[t]- V_{\rm th}).
\end{cases}
\end{aligned}\label{define-snn-B}
\end{equation}
Note that the layers other than the ($q+1$)-th layer are the same as in \eqref{define-snn}. Meanwhile, the forward process of the ($q+1$)-th layer of SAF is as follows: 
\begin{equation}
\begin{aligned}
\begin{cases}
\widehat{\bm{U}}^{q+1}[t] 
= \bm{W}^q \widehat{\bm{a}}^q[t] +\bm{b}^{q+1}(\sum_{\tau=0}^{t-1}\lambda^{\tau}) + \lambda^{t}\widehat{\bm{U}}^{q+1}[0]+\bm{W}_b \,\widehat{\bm{a}}^p[t-1],\vspace{2pt}\\
\widehat{\bm{a}}^{q+1}[t] = \lambda\widehat{\bm{a}}^{q+1}[t-1]+ H(\widehat{\bm{U}}^{q+1} -  V_{\rm th}(\lambda\widehat{\bm{a}}^{q+1}[t-1]+1)).
\end{cases}
\end{aligned}\label{SAF-Def-B}
\end{equation}
The layers other than the ($q+1$)-th layer are the same as in~\eqref{SAF-Def}. The forward processes of SAF and LIF with feedback connection are mutually convertible.\\
Regarding the backward processes, the gradients for parameters other than $\bm{W}_b$ are the same as when there is no feedback connection. The derivative with respect to $\bm{W}_b$ is calculated as
\begin{align}
\frac{\partial L_E[t]}{\partial \bm{W}_b}&=\widehat{\bm{a}}^p[t-1]\,\frac{\partial L_E[t]}{\partial \widehat{\bm{a}}^N[t]}\left(\prod_{i=N-1}^{q+1}\frac{\partial \widehat{\bm{a}}^{i+1}[t]}{\partial \widehat{\bm{a}}^{i}[t]}\right)\frac{\partial \widehat{\bm{a}}^{q+1}[t]}{\partial \widehat{\bm{U}}^{q+1}[t]}.
\end{align}
Therefore, we obtain $\partial L_E[t] / \partial \bm{W}_b=\widehat{\bm{a}}^p[t-1]\,\bm{g}_{\widehat{\bm{U}}}^{q+1}[t]$ for SAF-E and $\partial L_F / \partial \bm{W}_b=\widehat{\bm{a}}^p[T-1]\,\bm{g}_{\widehat{\bm{U}}}^{q+1}[T]$ for SAF-F. 
\subsubsection*{Equivalence with OTTT$_{\rm O}$ and Spike Representation}
We can show the equivalence of SAF-E and OTTT$_{\rm O}$ with a feedforward connection, or with a feedback connection, as well as in Theorem~\ref{thm1}. 
Moreover, as mentioned in Sec.~\ref{Training_methods_for_SNNs}, Spike Representation computes the gradients as \eqref{grad-SR}, then the equivalence of SAF-F and Spike Representation also holds, even with a feedforward connection.
\begin{cor} \label{corollary3}
For SNN with a feedforward connection \eqref{define-snn-F}, or a feedback connection \eqref{define-snn-B}, the following hold.
\item $\mathrm{(i)}$ The backward processes of SAF-E and OTTT$_{\rm O}$ with a feedforward connection are identical.
\item $\mathrm{(ii)}$ Suppose that $\bm{m}[t]$ converges when $t\rightarrow\infty$. Then, for sufficiently large $T$, the backward processes of SAF-F and Spike Representation with a feedforward connection are identical up to a scale factor.
\item $\mathrm{(iii)}$ The backward processes of SAF-E and OTTT$_{\rm O}$ with a feedback connection are identical. 
\end{cor}
The proof is stated in Appendix~\ref{subsec:Proof of Corollary3}.

\subsubsection*{Proximity to Spike Representation}
Assume that the SNN have a feedback connection same as in \eqref{define-snn-B}.  
Then, the same assertion as Theorem~\ref{thm2} does not hold, that is, SAF-F is not equivalent to Spike Representation in general. However, we can show that the gradient descent directions in SAF-F and Spike Representation are close.
\begin{thm}\label{thm3}
Suppose that $\bm{m}[t]$ converges when $t\rightarrow\infty$. Then, for sufficiently large $T$, the backward processes of SAF-F and Spike Representation with a feedback connection are similar, that is, $\left\langle \dfrac{\partial L_F}{\partial \bm{\theta}}, \left(\dfrac{\partial L_F}{\partial \bm{\theta}}\right)_{\rm SR}\right\rangle>0$ for all parameters $\bm{\theta}$.
\end{thm}
See Appendix \ref{subsec:Proof of Theorem3} for the proof. 

\section{Experiments}
\label{sec:Experiments}
In Sec.~\ref{subsec_equivalence}, we theoretically proved that SAF-E and OTTT$_{\rm O}$ as well as SAF-F and spike representation are equivalent. In this section, we experimentally compare these methods. As complex and large datasets make it difficult to analyze the results, we trained SAF on the CIFAR-10 and the CIFAR-100 datasets~\citep{krizhevsky2009learning} and inferred with SNN composed of LIF neurons. This experiment was performed five times with different initial parameters, and all approximation was executed by the sigmoid-like SG for fair comparison. We used the same experimental setup as \citep{Xiao2022OnlineNetworks}, including the choice of SG. 
The code was written in PyTorch~\citep{paszke2019pytorch}, and the experiments were executed using one GPU, an NVIDIA Tesla V100, 32GB. We show the pseudo-code of SAF-E and SAF-F in Algorithm~\ref{alg:saf} to better understand our methods. The implementation details are in Appendix~\ref{sec:implementation}. The main objective here is to analyze whether there are any inconsistencies between theory and experiment rather than to achieve state-of-the-art performance.
\clearpage

\begin{algorithm}[t]
\caption{One iteration of SAF training.}
\label{alg:saf}
\begin{algorithmic}[1]
\REQUIRE Network parameters $\{ \bm{W}^l \}$, $\{ \bm{b}^{l+1} \}$; Time steps $T$; Number of layers $N$; Other hyperparameters; Input dataset
\ENSURE Trained network parameters $\{ \bm{W}^l \}$, $\{ \bm{b}^{l+1} \}$
\FOR{$t = 1, 2, \ldots, T$} 
    \STATE \% Forward
    \FOR{$l = 1, 2, \ldots, N$} 
        \STATE Update the (weighted) potential accumulation $\widehat{\bm{U}}^l[t]$ and spike accumulation $\widehat{\bm{a}}^l[t]$ using \eqref{SAF-Def}.
    \ENDFOR
    \STATE \% Backward
    \FOR{$l = N, N-1, \ldots, 1$}
        \IF{training option is SAF-E}
            \STATE Update parameters with $\partial L_E[t] / \partial \bm{W}^l = \widehat{\bm{a}}^l[t]\,\bm{g}_{\widehat{\bm{U}}}^{l+1}[t]$ based on the gradient-based optimizer.
        \ELSIF{training option is SAF-F \AND $t=T$}
            \STATE Update parameters with $\partial L_F / \partial \bm{W}^l = \widehat{\bm{a}}^l[T]\,\bm{g}_{\widehat{\bm{U}}}^{l+1}[T]$ based on the gradient-based optimizer.
        \ENDIF
    \ENDFOR
\ENDFOR
\end{algorithmic}
\end{algorithm}

\subsection{Analysis of SAF-E}\label{subsec:SAF-E}

We experimentally analyze the performance of SAF-E. First, we compare accuracy. Table~\ref{table/result_E} shows the accuracy when we set $T=6$~(the difference of gradients are shown in Appendix~\ref{Comparison_of_gradient}). As shown in this table, SAF-E and OTTT$_{\rm O}$ have almost the same accuracy. The values in parentheses in Table~\ref{table/result_E} show the change in accuracy due to inference by an SNN composed of LIF neurons. From this table, we can see that the accuracy change due to inference by SNN consisting of LIF neurons is almost negligible in the case of CIFAR-10. On the other hand, CIFAR-100 shows a minor difference in accuracy (compared to the difference in gradient). This could be attributed to the increased task complexity, resulting in a more intricate loss function and greater susceptibility to even minor numerical errors affecting final accuracy.

Figure~\ref{fig/each_time_nofeedback} shows the accuracy and loss curves during the training of CIFAR-10. This indicates that the progress during training are comparable. Therefore, we confirmed experimentally that SAF-E and OTTT$_{\rm O}$ are numerically close. Similar results were obtained when there was a feedforward or a feedback connection~(see Appendix~\ref{appendix_experiment}). 

Next, we compare the training costs. From Table \ref{table/result_E}, it can be seen that SAF-E takes less time to train and uses less memory during training than OTTT$_{\rm O}$. However, OTTT$_{\rm O}$ can be executed with constant memory usage even as time steps increase. Therefore, we compared the training time and memory usage at different time steps. Figures~\ref{fig/training cost_only_each_time} (A) and (B) show the training time and memory at different time steps. Note that the training time was measured in one batch. It can be seen that the memory usage of SAF-E does not increase even if the number of time steps increases, similar to OTTT$_{\rm O}$. Also, from Fig.~\ref{fig/training cost_only_each_time} (B), we can see that SAF-E uses less memory than OTTT$_{\rm O}$. This result stems from the fact that SAF does not need to maintain the previous membrane potential.

\begin{table}[p]
  \caption{Performance comparison of SAF-E and OTTT$_{\rm O}$ on CIFAR-10 and CIFAR-100. The values in parentheses were the changes in accuracy and total firing rate due to inference by the SNN composed of LIF neurons. Note that training times were measured in one minibatch, and training time and memory were not perturbed between trials.}
  \begin{adjustbox}{center}
  \begin{tabular}{lccccc}   
   \multicolumn{2}{l}{\bf{CIFAR-10}}\\
   \hline
  \bf{Method}&\bf{$T$}&\bf{Memory~[GB]}&\bf{Training Time~[sec]} &\bf{Firing rate~[\%]}&\bf{Accuracy~[\%]}  \\
    \hline
    OTTT$_{\rm O}$&6&1.656&0.666&15.14$\pm$0.17&93.44$\pm$0.15 \\
    SAF-E  &6&1.184&0.468&14.76$\pm$0.15 (1.048$\times 10^{-5}$) &93.54$\pm$0.17 (0.016)\\    
    \hline\\

  \multicolumn{2}{l}{\bf{CIFAR-100}}\\
  \hline
  \bf{Method}&\bf{$T$}&\bf{Memory~[GB]}&\bf{Training Time~[sec]}  &\bf{Firing rate~[\%]}&\bf{Accuracy~[\%]}  \\
  \hline
    OTTT$_{\rm O}$&6&1.656&0.666&17.27$\pm$0.19&70.70$\pm$0.19 \\
    SAF-E  &6&1.186&0.464&16.77$\pm$0.08 (1.513$\times10^{-5}$) &71.56$\pm$0.35 (0.042)\\    
    \hline  \end{tabular}\label{table/result_E}
\end{adjustbox}
\end{table}

\begin{figure}[p]
\centering
  \begin{minipage}[b]{0.45\linewidth}
    \centering
\includegraphics[keepaspectratio, scale=0.36]{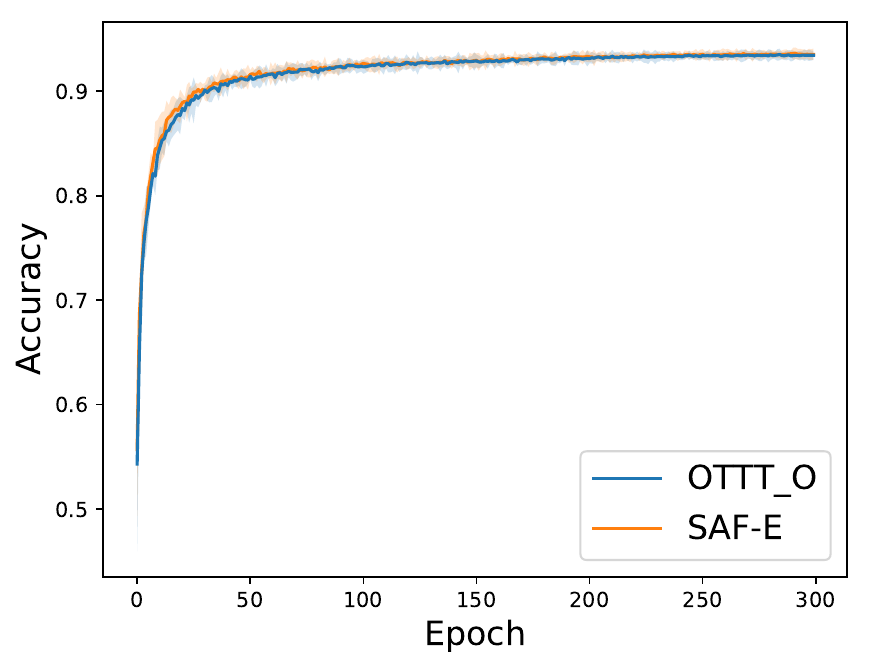} 

(A) Accuracy \label{acc_each_time}
  \end{minipage}
  \begin{minipage}[b]{0.45\linewidth}
    \centering
\includegraphics[keepaspectratio, scale=0.36]{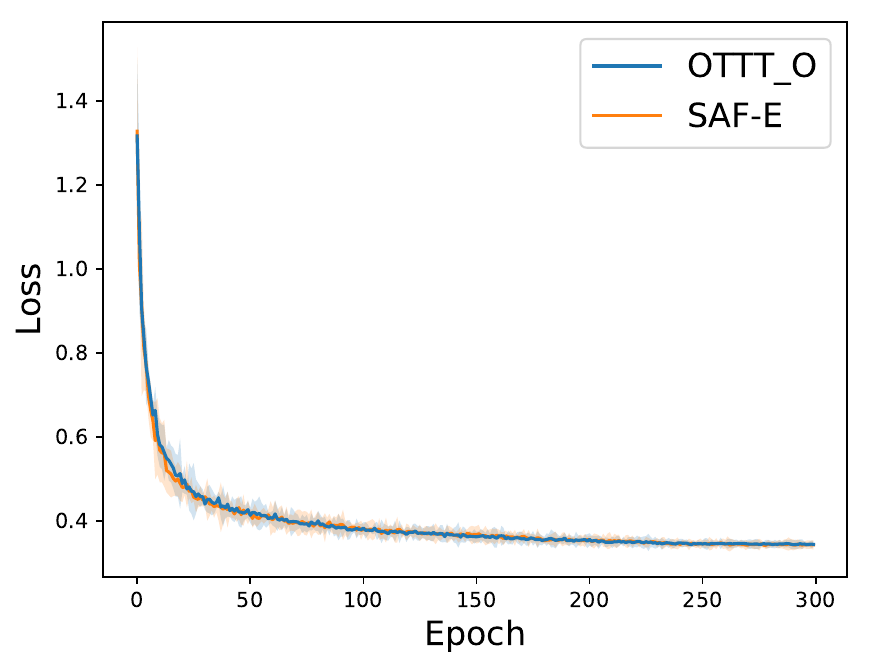}

(B) Loss
\label{loss_each_time}
  \end{minipage}
  \caption{Accuracy and loss curves of SAF-E and OTTT$_{\rm O}$ on CIFAR-10 ($T=6$).}\label{fig/each_time_nofeedback}
\end{figure}

\begin{figure}[p]
\begin{minipage}[t]
{0.33\linewidth}
    \centering
\includegraphics[keepaspectratio, scale=0.36]{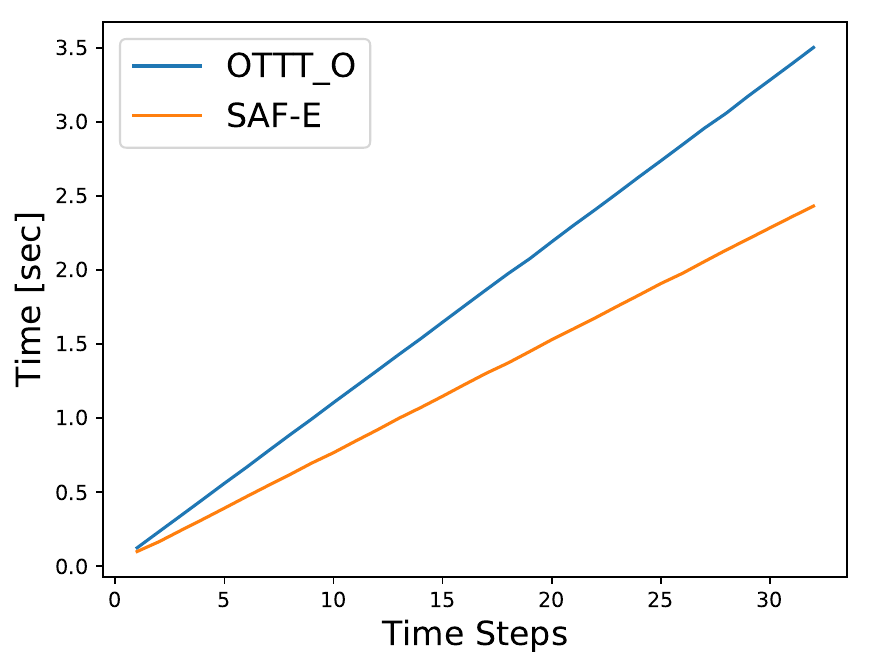}
\subcaption{(A) Training Time}
\label{fig/time_only_each_time}
\end{minipage}
\begin{minipage}[t]
{0.33\linewidth}
    \centering
\includegraphics[keepaspectratio, scale=0.36]{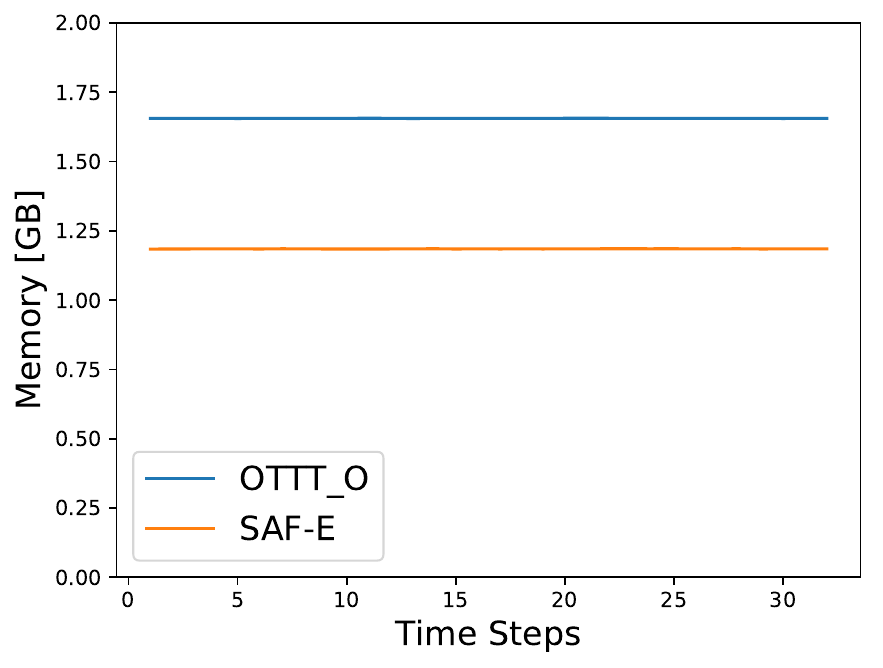}
\subcaption{(B) Memory Consumption}
\label{fig/memory_only_each_time}
  \end{minipage}
    \begin{minipage}[t]{0.33\linewidth}
    \centering
\includegraphics[keepaspectratio, scale=0.36]{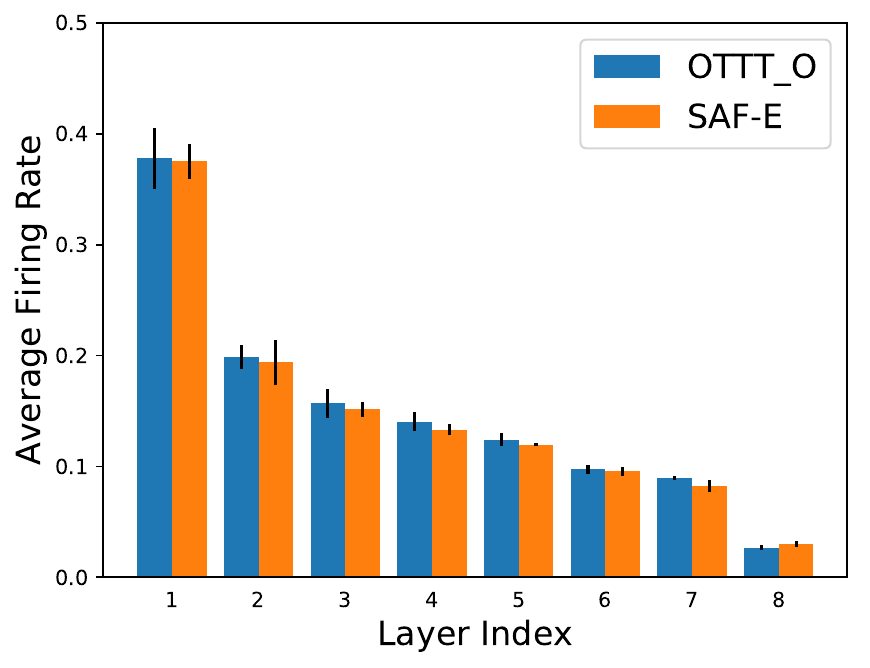}
\subcaption{(C) Firing rate}
\label{fig/firing_rate_only_each_time}
  \end{minipage}
  \caption{Training time, memory consumption, and firing rate of each layer of SAF-E and OTTT$_{\rm O}$ on CIFAR-10. 
  }\label{fig/training cost_only_each_time}
\end{figure}

Finally, we compare the firing rate. As shown in Table \ref{table/result_E}, the total firing rates of SAF-E and OTTT$_{\rm O}$ are close. Also, the amount of change due to inference with SNNs consisting of LIF neurons is also almost negligible, similar to the case for accuracy. Furthermore, from Fig.~\ref{fig/training cost_only_each_time} (C), the firing rates of each layer are almost close too.

These results indicate that using SAF-E can reduce the training time and memory compared to OTTT$_{\rm O}$ while achieving close firing rate and accuracy. It was also shown that using the parameters trained with SAF-E to infer with the SNN consisting of LIF neurons is feasible.

\subsection{Analysis of SAF-F}
\label{subsec:SAF-F} 
In this subsection, we experimentally analyze the performance of SAF-F. First, we compare accuracy. 
The results are shown in Table \ref{table/result_F} and Fig.~\ref{fig/full_time_nofeedback}~(the difference of gradients are shown in Appendix~\ref{Comparison_of_gradient}). Since the spike representation method is effective when $T$ is large, there is no theoretical guarantee that SAF-F can infer well in a short time step. However, Table~\ref{table/result_F} shows that the accuracy of SAF-F is almost the same for $T=6$ and $T=32$.
Note that Spike Representation methods, except for SAF-F, do not use SGs, which makes precise comparisons difficult. Therefore, we only compared SAF-F with OTTT$_{\rm A}$.

As with the previous results, the accuracy change due to inference with SNNs consisting of LIF neurons is almost negligible. Meanwhile, the accuracies of SAF-F and OTTT$_{\rm A}$ are close, though from the perspective of standard deviation, there seems to be a difference. From Sec.~\ref{subsec_equivalence}, the gradient directions of Spike Representation and SAF-F are identical, but those of Spike Representation and OTTT$_{\rm A}$ are only similar. Therefore, the gradient directions of SAF-F and OTTT$_{\rm A}$ are also only similar. This is thought to be the cause of the differences in accuracy and loss. 

Next, we compare the training costs. From Table~\ref{table/result_F}, it can be seen that SAF-F requires less time for training and uses less memory than OTTT$_{\rm A}$. This trend is also similar when the time step is varied~(see Figs.~\ref{fig/training cost_only_full} (A) and (B)). 

Finally, we compare firing rate. As shown in Table~\ref{table/result_F} as for accuracy, the change of the total firing rate by inferring with SNNs consisting of LIF neurons is almost negligible. Meanwhile, the total firing rate of SAF-F is smaller than of OTTT$_{\rm A}$. In addition, from Fig.~\ref{fig/training cost_only_full} (C), it can be seen that the firing rate of each layer~(especially the first layer) is smaller in SAF-F than OTTT$_{\rm A}$. These differences also indicate that SAF-F and OTTT$_{\rm A}$ are generally not identical. 

From the above analysis, we can say that SAF-F is a better choice than OTTT$_{\rm A}$ in terms of training time, memory usage, and firing rate. Also, as with SAF-E, inference can be performed by the standard SNN using the parameters trained by SAF-F.
\section{Limitation and Discussion}
In this paper, we have shown that SAF-E coincides with OTTT$_{\rm O}$ and that SAF-F coincides with OTTT$_{\rm A}$. On the other hand, \citet{Xiao2022OnlineNetworks} shows that OTTT methods are more accurate than BPTT. Given the concordance between SAF and OTTT, SAF is more accurate than BPTT. We consider that the better accuracy than BPTT is due to the difficulty of achieving optimal rollout at all times in BPTT, as known by vanishing gradients, and the fact that most of the datasets used in the field of deep SNNs are time-independent labels. Therefore, we assume that $\partial \widehat{{\bm a}}^N[t-1]/\partial {\bf W}^l$ is zero~(see Appendix~\ref{subsec:6_7}). This assumption is intended to even out the gradient's effect on parameter updates at each time, as described in Appendix~\ref{subsec:6_7}, and is also valid for widely used time-independent datasets. On the other hand, since SAF trains using information up to $t$ by the spike accumulation $\widehat{{\bm a}}$, not just the current time, SAF can implicitly train at each time while using information up to the previous time. The above assumption can be easily removed; however, its validation would require a labeled time-dependent data set. This would require the preparation of an appropriate data set, which is outside the scope of this paper, considering theoretical consistency with OTTT and is therefore considered one for future research.

\begin{table}[ht]
  \caption{Performance comparison of SAF-F and OTTT$_{\rm A}$ on CIFAR-10 and CIFAR-100.  The values in parentheses are the changes in accuracy and total firing rate due to inference by the SNN composed of LIF neurons. Note that training times were measured in one minibatch, and training time and memory were not perturbed between trials.}
  \begin{adjustbox}{center}
  \begin{tabular}{lccccc}
  \multicolumn{2}{l}{\bf{CIFAR-10}}\\
  \hline
  \bf{Method}&\bf{$T$}&\bf{Memory~[GB]}&\bf{Training Time~[sec]}  &\bf{Firing rate~[\%]}&\bf{Accuracy~[\%]}  \\
    \hline
    OTTT$_{\rm A}$&6&1.656&0.661&15.51$\pm$0.10&93.39$\pm$0.16 \\
    OTTT$_{\rm A}$&32&1.656&3.474&13.96$\pm$0.20&93.62$\pm$0.04 \\
    SAF-F  &6&1.157&0.247&10.50$\pm$0.19 (3.306$\times 10^{-5}$)&93.09$\pm$0.15 (0.076)\\ 
    SAF-F  &32&1.157&1.077&10.65$\pm$0.12 (0.965$\times10^{-5}$) &93.25$\pm$0.07 (0.002)\\    
    \hline\\
  \multicolumn{2}{l}{\bf{CIFAR-100}}\\
  \hline
  \bf{Method}&\bf{$T$}&\bf{Memory~[GB]}&\bf{Training Time~[sec]}  &\bf{Firing rate~[\%]}&\bf{Accuracy~[\%]}  \\
  \hline
    OTTT$_{\rm A}$&6&1.656&0.666&18.26$\pm$0.21&70.18$\pm$0.26 \\
    OTTT$_{\rm A}$&32&1.656&3.479&16.62$\pm$0.20&70.77$\pm$0.16 \\
    SAF-F  &6&1.157&0.249&12.19$\pm$0.22 (1.106$\times10^{-5}$) &70.58$\pm$0.34 (0.002)\\    
    SAF-F  &32&1.157&1.077&12.21$\pm$0.15 (1.016$\times10^{-5}$) &71.73$\pm$0.04 (0.096)\\    
    \hline  \end{tabular}\label{table/result_F}
\end{adjustbox}
\end{table}

\begin{figure}[t]
\centering
  \begin{minipage}[b]{0.45\linewidth}
    \centering
\includegraphics[keepaspectratio, scale=0.36]{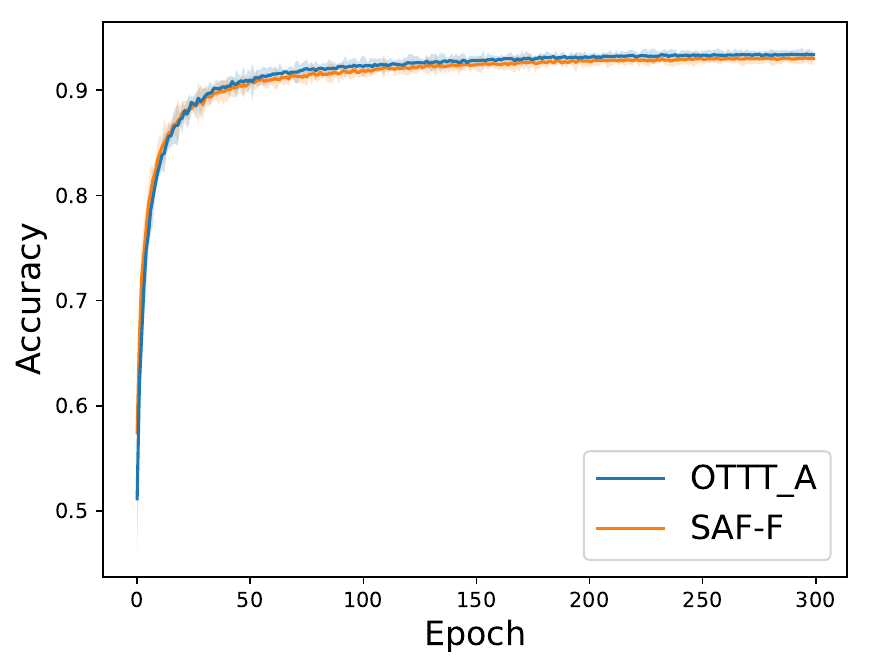} 

(A) Accuracy \label{acc_full}
  \end{minipage}
  \begin{minipage}[b]{0.45\linewidth}
    \centering
\includegraphics[keepaspectratio, scale=0.36]{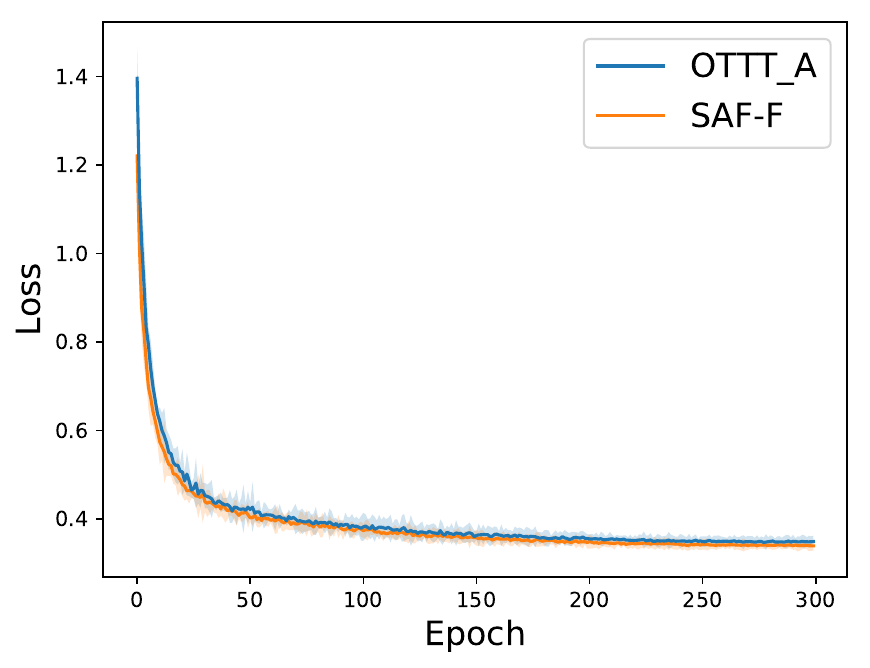}

(B) Loss
\label{loss_full}
  \end{minipage}
  \caption{Accuracy and loss curves of SAF-F and OTTT$_{\rm A}$ on CIFAR-10 ($T=6)$.}
  \label{fig/full_time_nofeedback}
\end{figure}

\begin{figure}[ht]
\begin{minipage}[t]
{0.33\linewidth}
    \centering
\includegraphics[keepaspectratio, scale=0.36]{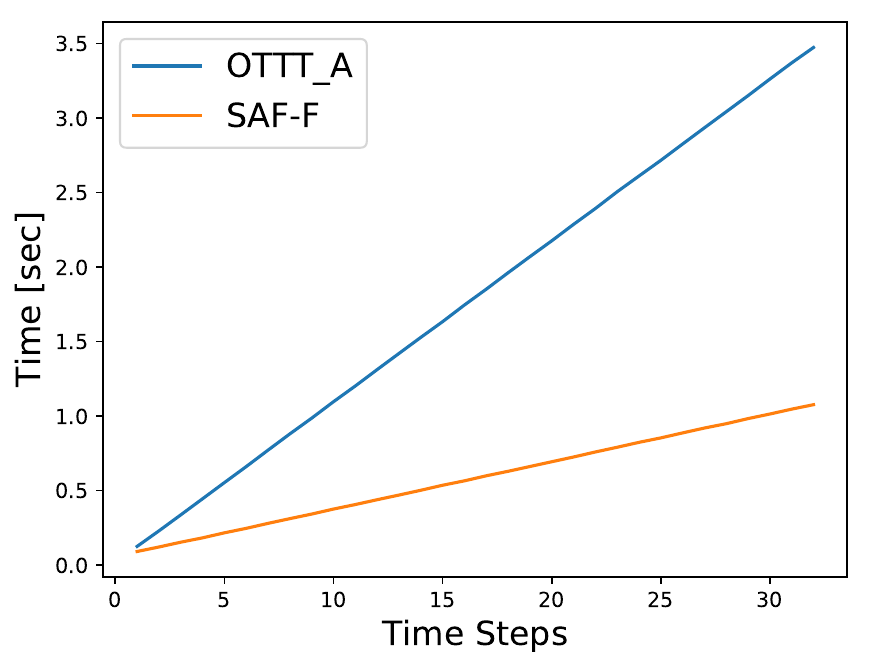}
\subcaption{(A) Training Time}
\label{fig/time_only_full}
\end{minipage}
\begin{minipage}[t]
{0.33\linewidth}
    \centering
\includegraphics[keepaspectratio, scale=0.36]{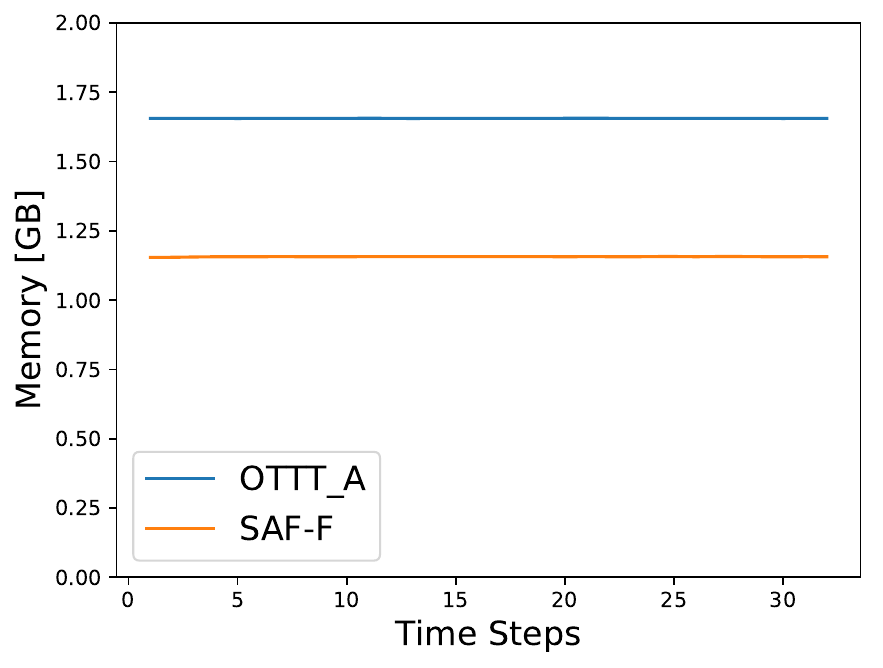}
\subcaption{(B) Memory Consumption}
\label{fig/memory_only_full}
  \end{minipage}
    \begin{minipage}[t]{0.33\linewidth}
    \centering
\includegraphics[keepaspectratio, scale=0.36]{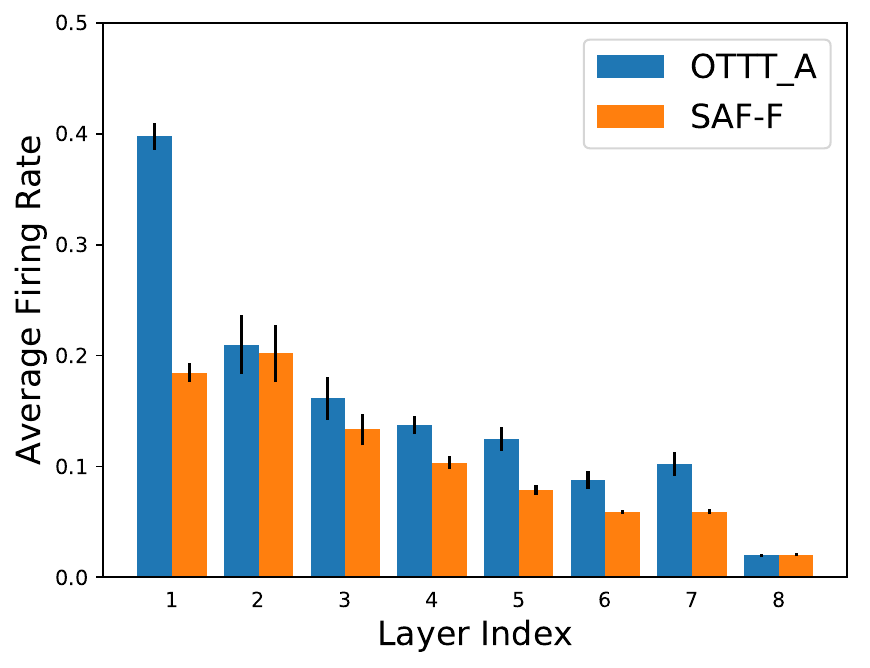}
\subcaption{(C) Firing rate}
\label{fig/firing_rate_only_full}
  \end{minipage}
  \caption{Training time, memory consumption, and firing rate of each layer of SAF-F and OTTT$_{\rm A}$ on CIFAR-10. 
  It can be seen that the firing rates do not match, which agree with the theoretical result that SAF-F is not always consistent with OTTT$_{\rm A}$.
  }\label{fig/training cost_only_full}
\end{figure}

\clearpage

\section{Conclusion and Future Work}
This article proposed SAF. SAF is a training method of SNNs that propagates the spike accumulation during training; however, SAF propagates the spike trains during inference, as do other SNNs. This article showed that SAF trained at each time step (SAF-E) is equivalent to OTTT$_{\rm O}$, and SAF trained at the final time step (SAF-F) is also equivalent to Spike Representation. We conducted experiments on the CIFAR-10 dataset and confirmed that the experimental results are consistent with these assertions and that training time and memory of SAF are reduced compared to OTTT. 

Most SNNs, including OTTT, Spike Representation, and SAF, train under time-independent labeled training data. Then, the issue of improving SAF to train even if the labels change at each time step remains for future work. Other remaining tasks are to experimentally confirm the similar results of this work for other surrogate gradients and to investigate the theoretical relationship with other learning rules such as SuperSpike.

SAF, as presented in this article, assumes training on a GPU. Therefore, it may not be suitable for training on neuromorphic chips. However, we believe executing training on GPUs and inference on neuromorphic chips is reasonable. In the future, we also plan to verify inference on the neuromorphic chip, including the extent to which numerical errors are affected by differences in computing environments.

\bibliographystyle{plainnat}
\bibliography{references_saf-ottt}

\begin{thebibliography}{38}
\providecommand{\natexlab}[1]{#1}
\providecommand{\url}[1]{\texttt{#1}}
\expandafter\ifx\csname urlstyle\endcsname\relax
  \providecommand{\doi}[1]{doi: #1}\else
  \providecommand{\doi}{doi: \begingroup \urlstyle{rm}\Url}\fi

\bibitem[Akopyan et~al.(2015)Akopyan, Sawada, Cassidy, Alvarez-Icaza, Arthur, Merolla, Imam, Nakamura, Datta, Nam, et~al.]{akopyan2015truenorth}
Filipp Akopyan, Jun Sawada, Andrew Cassidy, Rodrigo Alvarez-Icaza, John Arthur, Paul Merolla, Nabil Imam, Yutaka Nakamura, Pallab Datta, Gi-Joon Nam, et~al.
\newblock Truenorth: Design and tool flow of a 65 mw 1 million neuron programmable neurosynaptic chip.
\newblock \emph{IEEE transactions on computer-aided design of integrated circuits and systems}, 34\penalty0 (10):\penalty0 1537--1557, 2015.

\bibitem[Bengio et~al.(2015)Bengio, Mesnard, Fischer, Zhang, and Wu]{bengio2015stdp}
Yoshua Bengio, Thomas Mesnard, Asja Fischer, Saizheng Zhang, and Yuhuai Wu.
\newblock Stdp as presynaptic activity times rate of change of postsynaptic activity.
\newblock \emph{arXiv preprint arXiv:1509.05936}, 2015.

\bibitem[Bi and Poo(1998)]{bi1998synaptic}
Guo-qiang Bi and Mu-ming Poo.
\newblock Synaptic modifications in cultured hippocampal neurons: dependence on spike timing, synaptic strength, and postsynaptic cell type.
\newblock \emph{Journal of neuroscience}, 18\penalty0 (24):\penalty0 10464--10472, 1998.

\bibitem[Bohte et~al.(2002)Bohte, Kok, and La~Poutre]{bohte2002error}
Sander~M Bohte, Joost~N Kok, and Han La~Poutre.
\newblock Error-backpropagation in temporally encoded networks of spiking neurons.
\newblock \emph{Neurocomputing}, 48\penalty0 (1-4):\penalty0 17--37, 2002.

\bibitem[Chowdhury et~al.(2021)Chowdhury, Rathi, and Roy]{chowdhury2021one}
Sayeed~Shafayet Chowdhury, Nitin Rathi, and Kaushik Roy.
\newblock One timestep is all you need: Training spiking neural networks with ultra low latency.
\newblock \emph{arXiv preprint arXiv:2110.05929}, 2021.

\bibitem[Davies et~al.(2018)Davies, Srinivasa, Lin, Chinya, Cao, Choday, Dimou, Joshi, Imam, Jain, et~al.]{davies2018loihi}
Mike Davies, Narayan Srinivasa, Tsung-Han Lin, Gautham Chinya, Yongqiang Cao, Sri~Harsha Choday, Georgios Dimou, Prasad Joshi, Nabil Imam, Shweta Jain, et~al.
\newblock Loihi: A neuromorphic manycore processor with on-chip learning.
\newblock \emph{Ieee Micro}, 38\penalty0 (1):\penalty0 82--99, 2018.

\bibitem[Deng and Gu(2020)]{deng2020optimal}
Shikuang Deng and Shi Gu.
\newblock Optimal conversion of conventional artificial neural networks to spiking neural networks.
\newblock In \emph{International Conference on Learning Representations}, 2020.

\bibitem[Diehl et~al.(2015)Diehl, Neil, Binas, Cook, Liu, and Pfeiffer]{diehl2015fast}
Peter~U Diehl, Daniel Neil, Jonathan Binas, Matthew Cook, Shih-Chii Liu, and Michael Pfeiffer.
\newblock Fast-classifying, high-accuracy spiking deep networks through weight and threshold balancing.
\newblock In \emph{2015 International joint conference on neural networks (IJCNN)}, pages 1--8. IEEE, 2015.

\bibitem[Ding et~al.(2021)Ding, Zhang, Ma, Han, Ding, and Sun]{Xiaohan2021repvgg}
Xiaohan Ding, Xiangyu Zhang, Ningning Ma, Jungong Han, Guiguang Ding, and Jian Sun.
\newblock Repvgg: Making vgg-style convnets great again.
\newblock In \emph{Proceedings of the IEEE/CVF conference on computer vision and pattern recognition}, pages 13728 -- 13737, 2021.

\bibitem[Fr{\'e}maux and Gerstner(2016)]{fremaux2016neuromodulated}
Nicolas Fr{\'e}maux and Wulfram Gerstner.
\newblock Neuromodulated spike-timing-dependent plasticity, and theory of three-factor learning rules.
\newblock \emph{Frontiers in neural circuits}, 9:\penalty0 85, 2016.

\bibitem[Fung et~al.(2022)Fung, Heaton, Li, McKenzie, Osher, and Yin]{jfb2022}
Samy~Wu Fung, Howard Heaton, Qiuwei Li, Daniel McKenzie, Stanley Osher, and Wotao Yin.
\newblock Jfb:jacobian-free backpropagation for implicit networks.
\newblock In \emph{Proceedings of the AAAI Conference on Artificial Intelligence}, 2022.

\bibitem[Geng et~al.(2021)Geng, Zhang, Bai, Wang, and Lin]{OntrainingImplicitModels}
Zhengyang Geng, Xin-Yu Zhang, Shaojie Bai, Yisen Wang, and Zhouchen Lin.
\newblock On training implicit models.
\newblock In \emph{Advances in Neural Information Procesing Systems}, 2021.

\bibitem[Han et~al.(2020)Han, Srinivasan, and Roy]{han2020rmp}
Bing Han, Gopalakrishnan Srinivasan, and Kaushik Roy.
\newblock Rmp-snn: Residual membrane potential neuron for enabling deeper high-accuracy and low-latency spiking neural network.
\newblock In \emph{Proceedings of the IEEE/CVF conference on computer vision and pattern recognition}, pages 13558--13567, 2020.

\bibitem[Hebb(2005)]{hebb2005organization}
DO~Hebb.
\newblock \emph{The Organization of Behavior: A Neuropsychological Theory}.
\newblock Psychology Press, 2005.

\bibitem[Kaiser et~al.(2020)Kaiser, Mostafa, and Neftci]{kaiser2020synaptic}
Jacques Kaiser, Hesham Mostafa, and Emre Neftci.
\newblock Synaptic plasticity dynamics for deep continuous local learning (decolle).
\newblock \emph{Frontiers in Neuroscience}, 14:\penalty0 515306, 2020.

\bibitem[Kim et~al.(2020)Kim, Park, Na, and Yoon]{kim2020spiking}
Seijoon Kim, Seongsik Park, Byunggook Na, and Sungroh Yoon.
\newblock Spiking-yolo: spiking neural network for energy-efficient object detection.
\newblock In \emph{Proceedings of the AAAI conference on artificial intelligence}, volume~34, pages 11270--11277, 2020.

\bibitem[Krizhevsky and Hinton(2009)]{krizhevsky2009learning}
Alex Krizhevsky and Geoffrey Hinton.
\newblock Learning multiple layers of features from tiny images.
\newblock Technical report, 2009.
\newblock URL \url{https://www.cs.toronto.edu/~kriz/learning-features-2009-TR.pdf}.

\bibitem[Lapique(1907)]{lapique1907researches}
Louis Lapique.
\newblock Researches quantatives sur l'excitation electrique des nerfs traitee comme une polarization.
\newblock \emph{Journal of Physiology, Pathology and Genetics}, 9:\penalty0 620--635, 1907.

\bibitem[Lian et~al.(2023)Lian, Shen, Liu, Wang, Yan, and Tang]{ijcai2023p335}
Shuang Lian, Jiangrong Shen, Qianhui Liu, Ziming Wang, Rui Yan, and Huajin Tang.
\newblock Learnable surrogate gradient for direct training spiking neural networks.
\newblock In \emph{Proceedings of the Thirty-Second International Joint Conference on Artificial Intelligence, {IJCAI-23}}, pages 3002--3010, 8 2023.

\bibitem[Luo et~al.(2023)Luo, Wong, Goh, Do, Chen, Li, Jiang, and Yau]{luo2023achieving}
Tao Luo, Weng-Fai Wong, Rick Siow~Mong Goh, Anh~Tuan Do, Zhixian Chen, Haizhou Li, Wenyu Jiang, and Weiyun Yau.
\newblock Achieving green ai with energy-efficient deep learning using neuromorphic computing.
\newblock \emph{Communications of the ACM}, 66\penalty0 (7):\penalty0 52--57, 2023.

\bibitem[Meng et~al.(2022)Meng, Xiao, Yan, Wang, Lin, and Luo]{meng2022training}
Qingyan Meng, Mingqing Xiao, Shen Yan, Yisen Wang, Zhouchen Lin, and Zhi-Quan Luo.
\newblock Training high-performance low-latency spiking neural networks by differentiation on spike representation.
\newblock In \emph{Proceedings of the IEEE/CVF Conference on Computer Vision and Pattern Recognition}, pages 12444--12453, 2022.

\bibitem[Neftci et~al.(2019)Neftci, Mostafa, and Zenke]{neftci2019surrogate}
Emre~O Neftci, Hesham Mostafa, and Friedemann Zenke.
\newblock Surrogate gradient learning in spiking neural networks: Bringing the power of gradient-based optimization to spiking neural networks.
\newblock \emph{IEEE Signal Processing Magazine}, 36\penalty0 (6):\penalty0 51--63, 2019.

\bibitem[Paszke et~al.(2019)Paszke, Gross, Massa, Lerer, Bradbury, Chanan, Killeen, Lin, Gimelshein, Antiga, et~al.]{paszke2019pytorch}
Adam Paszke, Sam Gross, Francisco Massa, Adam Lerer, James Bradbury, Gregory Chanan, Trevor Killeen, Zeming Lin, Natalia Gimelshein, Luca Antiga, et~al.
\newblock Pytorch: An imperative style, high-performance deep learning library.
\newblock \emph{Advances in neural information processing systems}, 32, 2019.

\bibitem[Qiao et~al.(2019)Qiao, Wang, Liu, Shen, and Yuille]{qiao2019micro}
Siyuan Qiao, Huiyu Wang, Chenxi Liu, Wei Shen, and Alan Yuille.
\newblock Micro-batch training with batch-channel normalization and weight standardization.
\newblock \emph{arXiv preprint arXiv:1903.10520}, 2019.

\bibitem[Qu et~al.(2023)Qu, Gao, Zhang, Lu, Tang, and Qiao]{qu2023spiking}
Jinye Qu, Zeyu Gao, Tielin Zhang, Yanfeng Lu, Huajin Tang, and Hong Qiao.
\newblock Spiking neural network for ultra-low-latency and high-accurate object detection.
\newblock \emph{arXiv preprint arXiv:2306.12010}, 2023.

\bibitem[Semedo et~al.(2022)Semedo, Jasper, Zandvakili, Krishna, Aschner, Machens, Kohn, and Yu]{semedo2022feedforward}
Jo{\~a}o~D Semedo, Anna~I Jasper, Amin Zandvakili, Aravind Krishna, Amir Aschner, Christian~K Machens, Adam Kohn, and Byron~M Yu.
\newblock Feedforward and feedback interactions between visual cortical areas use different population activity patterns.
\newblock \emph{Nature communications}, 13\penalty0 (1):\penalty0 1099, 2022.

\bibitem[Shrestha and Orchard(2018)]{shrestha2018slayer}
Sumit~B Shrestha and Garrick Orchard.
\newblock Slayer: Spike layer error reassignment in time.
\newblock \emph{Advances in neural information processing systems}, 31, 2018.

\bibitem[Stein(1965)]{stein1965theoretical}
Richard~B Stein.
\newblock A theoretical analysis of neuronal variability.
\newblock \emph{Biophysical Journal}, 5\penalty0 (2):\penalty0 173--194, 1965.

\bibitem[Suetake et~al.(2023)Suetake, ichi Ikegawa, Saiin, and Sawada]{Suetake2022S3NN:Networks}
Kazuma Suetake, Shin ichi Ikegawa, Ryuji Saiin, and Yoshihide Sawada.
\newblock {S$^3$NN}: Time step reduction of spiking surrogate gradients for training energy efficient single-step spiking neural networks.
\newblock \emph{Neural Networks}, 159:\penalty0 208--219, 2023.

\bibitem[Thiele et~al.(2019)Thiele, Bichler, and Dupret]{thiele2019spikegrad}
Johannes~C Thiele, Olivier Bichler, and Antoine Dupret.
\newblock Spikegrad: An ann-equivalent computation model for implementing backpropagation with spikes.
\newblock In \emph{International Conference on Learning Representations}, 2019.

\bibitem[Wu et~al.(2021)Wu, Chua, Zhang, Li, Li, and Tan]{wu2021tandem}
Jibin Wu, Yansong Chua, Malu Zhang, Guoqi Li, Haizhou Li, and Kay~Chen Tan.
\newblock A tandem learning rule for effective training and rapid inference of deep spiking neural networks.
\newblock \emph{IEEE Transactions on Neural Networks and Learning Systems}, 2021.

\bibitem[Wu et~al.(2018)Wu, Deng, Li, Zhu, and Shi]{wu2018spatio}
Yujie Wu, Lei Deng, Guoqi Li, Jun Zhu, and Luping Shi.
\newblock Spatio-temporal backpropagation for training high-performance spiking neural networks.
\newblock \emph{Frontiers in neuroscience}, 12:\penalty0 331, 2018.

\bibitem[Xiao et~al.(2021)Xiao, Meng, Zhang, Wang, and Lin]{Xiao2021TrainingState}
Mingqing Xiao, Qingyan Meng, Zongpeng Zhang, Yisen Wang, and Zhouchen Lin.
\newblock {Training Feedback Spiking Neural Networks by Implicit Differentiation on the Equilibrium State}.
\newblock \emph{Advances in Neural Information Processing Systems}, 18\penalty0 (NeurIPS):\penalty0 14516--14528, 2021.

\bibitem[Xiao et~al.(2022)Xiao, Meng, Zhang, He, and Lin]{Xiao2022OnlineNetworks}
Mingqing Xiao, Qingyan Meng, Zongpeng Zhang, Di~He, and Zhouchen Lin.
\newblock Online training through time for spiking neural networks.
\newblock \emph{Advances in Neural Information Processing Systems}, 35:\penalty0 20717--20730, 2022.

\bibitem[Xiao et~al.(2023)Xiao, Meng, Zhang, Wang, and Lin]{Xiao2023SPIDE:Networks}
Mingqing Xiao, Qingyan Meng, Zongpeng Zhang, Yisen Wang, and Zhouchen Lin.
\newblock {SPIDE}: A purely spike-based method for training feedback spiking neural networks.
\newblock \emph{Neural Networks}, 161:\penalty0 9--24, 2023.

\bibitem[Zenke and Ganguli(2018)]{zenke2018superspike}
Friedemann Zenke and Surya Ganguli.
\newblock Superspike: Supervised learning in multilayer spiking neural networks.
\newblock \emph{Neural computation}, 30\penalty0 (6):\penalty0 1514--1541, 2018.

\bibitem[Zheng et~al.(2021)Zheng, Wu, Deng, Hu, and Li]{zheng2021going}
Hanle Zheng, Yujie Wu, Lei Deng, Yifan Hu, and Guoqi Li.
\newblock Going deeper with directly-trained larger spiking neural networks.
\newblock In \emph{Proceedings of the AAAI conference on artificial intelligence}, volume~35, pages 11062--11070, 2021.

\bibitem[Zhou et~al.(2021)Zhou, Li, Chen, Chandrasekaran, and Sanyal]{zhou2021temporal}
Shibo Zhou, Xiaohua Li, Ying Chen, Sanjeev~T Chandrasekaran, and Arindam Sanyal.
\newblock Temporal-coded deep spiking neural network with easy training and robust performance.
\newblock In \emph{Proceedings of the AAAI conference on artificial intelligence}, volume~35, pages 11143--11151, 2021.

\end{thebibliography}

\newpage

\appendix

\part*{Appendix}
\section{List of main Formulas}\label{list_of_main_formulas}
The neurons and gradients are as follows.
Note that $\bm{s}^l[t]$, $\bm{u}^l[t]$, $\bm{W}^l$ and $\bm{b}^l$ are the spike train, membrane potential, weight, and bias of $l$-th~(also denoted by $p$-th and $q$-th) layer, respectively.
Also, $\lambda \le 1$ is the leaky term, $V_{\rm th}$ is the threshold, $L$, $L_E$, and $L_F$ are the loss functions, $H$ is the element-wise Heaviside step function, $N$ is the number of layers, $\bm{a}[t] = \sum_{\tau = 0}^{t}\lambda^{t-\tau}\bm{s}[\tau] / \sum_{\tau = 0}^{t}\lambda^{t-\tau}$ is the weighted firing rate, $\widehat{\bm{a}}[t] = \sum_{\tau = 0}^{t}\lambda^{t-\tau}\bm{s}[\tau]$ is the (weighted) spike accumulation, and $ \widehat{\bm{U}}^{l+1}[t] = \lambda\widehat{\bm{U}}^{l+1}[t-1] + \bm{W}^l(\widehat{\bm{a}}^l[t]-\lambda\widehat{\bm{a}}^l[t-1]) + \bm{b}^{l+1}$ is the (weighted) potential accumulation.

\subsection{Neurons}
LIF neuron \eqref{define-snn}:
\begin{align*}
\begin{cases}
\bm{u}^{l+1}[t] = \lambda(\bm{u}^{l+1}[t-1] - V_{\rm th} \, \bm{s}^{l+1}[t-1]) + \bm{W}^l \bm{s}^{l}[t] + \bm{b}^{l+1},\\
\bm{s}^{l+1}[t] = H(\bm{u}^{l+1}[t] - V_{\rm th}).
\end{cases}
\end{align*}

SAF neuron \eqref{SAF-Def}:
\begin{align*}
\begin{cases}
\widehat{\bm{U}}^{l+1}[t] = \bm{W}^l \widehat{\bm{a}}^l[t] + \bm{b}^{l+1}\sum_{\tau = 0}^{t-1}\lambda^{t-\tau} + \lambda^{t}\widehat{\bm{U}}^{l+1}[0],\vspace{2pt}\\
\widehat{\bm{a}}^{l+1}[t] = \lambda\widehat{\bm{a}}^{l+1}[t-1] + H(\widehat{\bm{U}}^{l+1}[t] -  V_{\rm th}(\lambda\widehat{\bm{a}}^{l+1}[t-1] +1)).
\end{cases}
\end{align*}

\subsection{Neurons with feedforward connection}
LIF neuron \eqref{define-snn-F}:
\begin{align*}
\begin{cases}
\bm{u}^{q+1}[t] = \lambda(\bm{u}^{q+1}[t-1] - V_{\rm th}\,\bm{s}^{q+1}[t-1]) + \bm{W}^q \bm{s}^{q}[t] + \bm{b}^{q+1}+\bm{W}_f\bm{s}^p[t],\\
\bm{s}^{q+1}[t] = H(\bm{u}^{q+1}[t]- V_{\rm th}).
\end{cases}
\end{align*}

SAF neuron \eqref{SAF-Def-F}:
\begin{align*}
\begin{cases}
\widehat{\bm{U}}^{q+1}[t]= \bm{W}^q \widehat{\bm{a}}^q[t] +\bm{b}^{q+1}(\sum_{\tau=0}^{t-1}\lambda^{\tau}) + \lambda^{t}\widehat{\bm{U}}^{q+1}[0]+\bm{W}_f \, \widehat{\bm{a}}^p[t],\\
\widehat{\bm{a}}^{q+1}[t] = \lambda\widehat{\bm{a}}^{q+1}[t-1]+ H(\widehat{\bm{U}}^{q+1} -  V_{\rm th}(\lambda\widehat{\bm{a}}^{q+1}[t-1]+1)).
\end{cases}
\end{align*}

\subsection{Neurons with feedback connection}
LIF neuron \eqref{define-snn-B}:
\begin{align*}
\begin{cases}
\bm{u}^{q+1}[t] = \lambda(\bm{u}^{q+1}[t-1] - V_{\rm th} \,\bm{s}^{q+1}[t-1]) + \bm{W}^q \bm{s}^{q}[t] + \bm{b}^{q+1}+\bm{W}_b \,\bm{s}^p[t-1],\\
\bm{s}^{q+1}[t] = H(\bm{u}^{q+1}[t]- V_{\rm th}).
\end{cases}
\end{align*}

SAF neuron \eqref{SAF-Def-B}:
\begin{align*}
\begin{cases}
\widehat{\bm{U}}^{q+1}[t] 
= \bm{W}^q \widehat{\bm{a}}^q[t] +\bm{b}^{q+1}(\sum_{\tau=0}^{t-1}\lambda^{\tau}) + \lambda^{t}\widehat{\bm{U}}^{q+1}[0]+\bm{W}_b \,\widehat{\bm{a}}^p[t-1],\vspace{2pt}\\
\widehat{\bm{a}}^{q+1}[t] = \lambda\widehat{\bm{a}}^{q+1}[t-1]+ H(\widehat{\bm{U}}^{q+1} -  V_{\rm th}(\lambda\widehat{\bm{a}}^{q+1}[t-1]+1)).
\end{cases}
\end{align*}

\subsection{Gradients}

Spike Representation \eqref{grad-SR}:
\begin{equation*}
\left(\frac{\pd L}{\pd \bm{W}^l}\right)_{\rm SR} = \frac{\pd L}{\pd \bm{a}^{N}[T]}\left(\prod_{i=N-1}^{l+1}\frac{\pd \bm{a}^{i+1}[T]}{\partial \bm{a}^{i}[T]}\right)\frac{\pd\bm{a}^{l+1}[T]}{\pd\bm{W}^l}.
\end{equation*}

OTTT$_{\rm O}$ \eqref{grad-OTTT}:
\begin{equation*}
\left(\frac{\pd L[t]}{\pd \bm{W}^l} \right)_{\rm OT}= \widehat{\bm{a}}^l[t]\, \frac{\pd L[t]}{\pd \bm{s}^{N}[t]}\left(\prod_{i=N-1}^{l+1}\frac{\pd \bm{s}^{i+1}[t]}{\partial \bm{s}^{i}[t]}\right)\frac{\pd \bm{s}^{l+1}[t]}{\pd \bm{u}^{l+1}[t]}.
\end{equation*}

OTTT$_{\rm A}$:
\begin{equation*}
\left(\frac{\pd L}{\pd \bm{W}^l} \right)_{\rm OT}= \sum_{t}^{T}\widehat{\bm{a}}^l[t]\, \frac{\pd L[t]}{\pd \bm{s}^{N}[t]}\left(\prod_{i=N-1}^{l+1}\frac{\pd \bm{s}^{i+1}[t]}{\partial \bm{s}^{i}[t]}\right)\frac{\pd \bm{s}^{l+1}[t]}{\pd \bm{u}^{l+1}[t]}.
\end{equation*}

SAF-E \eqref{grad-SAF-o2}, or \eqref{gr-SAF-E}:
\begin{align*}
\frac{\partial L_E[t]}{\partial \bm{W}^l}&=\widehat{\bm{a}}^l[t]\,\frac{\partial L_E[t]}{\partial \widehat{\bm{a}}^N[t]}\left(\prod_{i=N-1}^{l+1}\frac{\partial \widehat{\bm{a}}^{i+1}[t]}{\partial \widehat{\bm{a}}^{i}[t]}\right)\frac{\partial \widehat{\bm{a}}^{l+1}[t]}{\partial \widehat{\bm{U}}^{l+1}[t]}. 
\end{align*}

SAF-F \eqref{gr-SAF-F}:
\begin{align*}
\frac{\partial L_F}{\partial \bm{W}^l}&=\widehat{\bm{a}}^l[T]\,\frac{\partial L_F}{\partial \widehat{\bm{a}}^N[T]}\left(\prod_{i=N-1}^{l+1}\frac{\partial \widehat{\bm{a}}^{i+1}[T]}{\partial \widehat{\bm{a}}^{i}[T]}\right)\frac{\partial \widehat{\bm{a}}^{l+1}[T]}{\partial \widehat{\bm{U}}^{l+1}[T]}. 
\end{align*}

\section{Derivation and Proofs}
\subsection{Derivation of (\ref{SAF-Def}) and (\ref{eq:SAF-LIF})} \label{subsection_derivation_4_5}

In this subsection, we derive \eqref{SAF-Def} and \eqref{eq:SAF-LIF} through proving that \eqref{SAF-Def} and \eqref{eq:SAF-LIF} hold for any $t \in \{1,\ldots,T\}$ with mathematical induction.

First, we prove that \eqref{SAF-Def} and \eqref{eq:SAF-LIF} hold for $t =1$.
We compute $\widehat{\bm{U}}^{l+1}[1]$ based on definition as follows:
\begin{align*}
\widehat{\bm{U}}^{l+1}[1] &=  \lambda\widehat{\bm{U}}^{l+1}[0] + \bm{W}^l(\widehat{\bm{a}}^l[1]-\lambda\widehat{\bm{a}}^l[0]) + \bm{b^{l+1}}\\
&=\bm{W}^l\widehat{\bm{a}}^l[1] + \bm{b}^{l+1} + \lambda\widehat{\bm{U}}^{l+1}[0].
\end{align*}
Taking account into $\widehat{\bm{a}}^{l+1}[0] = \widehat{\bm{s}}^{l+1}[0] = 0$, we obtain
\begin{align*}
\bm{u}^{l+1}[1]  &= \lambda(\bm{u}^{l+1}[0] - V_{\rm th}\,\bm{s}^{l+1}[0]) + \bm{W}^l\bm{s}^l[1] + \bm{b^{l+1}}\\
&=\widehat{\bm{U}}^{l+1}[0] - V_{\rm th}\,\widehat{\bm{a}}^{l+1}[0]) + \bm{W}^l(\widehat{\bm{a}}^l[1]-\lambda\widehat{\bm{a}}^l[0]) + \bm{b^{l+1}}\\
&= \widehat{\bm{U}}^{l+1}[1] -  V_{\rm th}\,\widehat{\bm{a}}^{l+1}[0].
\end{align*}
With this equation, we have
\begin{align*}
\bm{s}^{l+1}[1] &= H(\bm{u}^{l+1}[1] - V_{\rm th})\\
&= H(\widehat{\bm{U}}^{l+1}[1] -  V_{\rm th}(\lambda\widehat{\bm{a}}^{l+1}[0]+1)).
\end{align*}
Then, $\widehat{\bm{a}}^{l+1}[1]$ can be computed as follows:
\begin{align*}
\widehat{\bm{a}}^{l+1}[1] &= \sum_{\tau = 0}^{1}\lambda^{1-\tau}\bm{s}^{l+1}[\tau]\\
&= \lambda \widehat{\bm{a}}^{l+1}[0] + \bm{s}^{l+1}[1]\\
&= \lambda \widehat{\bm{a}}^{l+1}[0] + H(\widehat{\bm{U}}^{l+1}[1] -  V_{\rm th}(\lambda\widehat{\bm{a}}^{l+1}[0]+1)).
\end{align*}
Therefore, \eqref{SAF-Def} and \eqref{eq:SAF-LIF} hold for $t = 1$.

Next, we prove that \eqref{SAF-Def} and \eqref{eq:SAF-LIF} hold for $t$ when they hold for any $\tau \in \{1,\ldots,t-1\}$. Assuming that \eqref{SAF-Def} and \eqref{eq:SAF-LIF} hold for any $\tau \in \{1,\ldots,t-1\}$, we have
\begin{align*}
\widehat{\bm{U}}^{l+1}[t] &=  \lambda\widehat{\bm{U}}^{l+1}[t-1] + \bm{W}^l(\widehat{\bm{a}}^l[t]-\lambda\widehat{\bm{a}}^l[t-1]) + \bm{b^{l+1}}\\
&= \lambda\left(\bm{W}^l \widehat{\bm{a}}^l[t-1] + \bm{b}^{l+1}\sum_{\tau = 0}^{t-2}\lambda^{t-\tau} + \lambda^{t-1}\widehat{\bm{U}}^{l+1}[0]\right) + \bm{W}^l(\widehat{\bm{a}}^l[t]-\lambda\widehat{\bm{a}}^l[t-1]) + \bm{b^{l+1}}\\
&= \bm{W}^l \widehat{\bm{a}}^l[t] + \bm{b}^{l+1}\sum_{\tau = 0}^{t-1}\lambda^{t-\tau} + \lambda^{t}\widehat{\bm{U}}^{l+1}[0].
\end{align*}
Also, $\bm{u}^{l+1}[t]$ is computed as follows:
\begin{align*}
\bm{u}^{l+1}[t] &= \lambda(\bm{u}^{l+1}[t-1] - V_{\rm th}\,\bm{s}^{l+1}[t-1]) + \bm{W}^l\bm{s}^l[t] + \bm{b^{l+1}}\\
&= \lambda(\widehat{\bm{U}}^{l+1}[t-1] - V_{\rm th}(\lambda\widehat{\bm{a}}^{l+1}[t-2] +\bm{s}^{l+1}[t-1]) )+ \bm{W}^l(\widehat{\bm{a}}^l[t]-\lambda\widehat{\bm{a}}^l[t-1]) + \bm{b^{l+1}}\\
&= \widehat{\bm{U}}^{l+1}[t] - V_{\rm th}\lambda\widehat{\bm{a}}^{l+1}[t-1].
\end{align*}
With this equation, we have
\begin{align*}
\bm{s}^{l+1}[t] &=  H(\bm{u}^{l+1}[t] - V_{\rm th})\\
&=  H(\widehat{\bm{U}}^{l+1}[t] -  V_{\rm th}(\lambda\widehat{\bm{a}}^{l+1}[t-1]+1)).
\end{align*}
Then, $\widehat{\bm{a}}^{l+1}[t]$ can be computed as follows:
\begin{align*}
\widehat{\bm{a}}^{l+1}[t] &= \sum_{\tau = 0}^{t}\lambda^{t-\tau}\bm{s}^{l+1}[\tau]\\
&= \lambda \widehat{\bm{a}}^{l+1}[t-1] + \bm{s}^{l+1}[t]\\
&= \lambda \widehat{\bm{a}}^{l+1}[t-1] + H(\widehat{\bm{U}}^{l+1}[t] -  V_{\rm th}(\lambda\widehat{\bm{a}}^{l+1}[t-1]+1)).
\end{align*}
Hence, \eqref{SAF-Def} and \eqref{eq:SAF-LIF} hold for $t$ when they hold for any $\tau \in \{1,\ldots,t-1\}$.

Therefore, \eqref{SAF-Def} and \eqref{eq:SAF-LIF} hold any $t \in \{1,\ldots,T\}$.

\subsection{Derivation of (\ref{gr-SAF-E}) and (\ref{gr-SAF-F})}\label{subsec:6_7}
First, since $L_E[t]=\mathcal{L}(\bm{s}^N[t],\bm{y})/T$, we have
\begin{align}
\frac{\partial L_E[t]}{\partial \bm{W}^l}&=\frac{\partial L_E[t]}{\partial \bm{s}^N[t]}\frac{\partial \bm{s}^N[t]}{\partial \bm{W}^l}.
\end{align}
Assuming that $L_E[t]$ depends only on $\widehat{\bm{a}}^l[t]$ and $\widehat{\bm{U}}^l[t]$, i.e., not on anything up to $t-1$, we regard $\widehat{\bm{a}}^N[t]$  as $\bm{s}^N[t]+\rm{Const}$. Then, we calculate that
\begin{align}
\frac{\partial L_E[t]}{\partial \bm{W}^l}&=\frac{\partial L_E[t]}{\partial \bm{s}^N[t]}\frac{\partial \bm{s}^N[t]}{\partial \widehat{\bm{a}}^N[t]}\frac{\partial \widehat{\bm{a}}^N[t]}{\partial \bm{W}^l}=\frac{\partial L_E[t]}{\partial \widehat{\bm{a}}^N[t]}\frac{\partial \widehat{\bm{a}}^N[t]}{\partial \bm{W}^l}.\label{A.2-1}
\end{align}
Next, it follows from \eqref{SAF-Def} that  
\begin{align}
\frac{\partial \widehat{\bm{a}}^N[t]}{\partial \bm{W}^l}&=\frac{\partial \widehat{\bm{a}}^N[t]}{\partial \widehat{\bm{a}}^N[t-1]}\frac{\partial \widehat{\bm{a}}^N[t-1]}{\partial \bm{W}^l}+\frac{\partial \widehat{\bm{a}}^N[t]}{\partial \widehat{\bm{U}}^N[t]}\frac{\partial \widehat{\bm{U}}^N[t]}{\partial \bm{W}^l}\label{A.2-2}\\ 
&=\frac{\partial \widehat{\bm{a}}^N[t]}{\partial \widehat{\bm{U}}^N[t]}\frac{\partial \widehat{\bm{U}}^N[t]}{\partial \widehat{\bm{a}}^{N-1}[t]}\frac{\partial \widehat{\bm{a}}^{N-1}[t]}{\partial \bm{W}^l}\\
&=\frac{\partial \widehat{\bm{a}}^N[t]}{\partial \widehat{\bm{a}}^{N-1}[t]}\frac{\partial \widehat{\bm{a}}^{N-1}[t]}{\partial \bm{W}^l},\label{A.2-3}
\end{align}
where we regard $\partial \widehat{\bm{a}}^N[t-1]/\partial \bm{W}^l$ as $0$, and $\widehat{\bm{U}}^N[t]$ as a function of $\widehat{\bm{a}}^{N-1}[t]$. It should be noted that the assumption $\partial \widehat{\bm{a}}[t-1]/\partial \bm{W}=0$ implies equalizing the effect of gradient on parameter updates at each time. Indeed, $\partial \widehat{\bm{a}}[t-1] / \partial \bm{W}$ is already used for updating parameters at time $t-1$. Therefore, using also $\partial \widehat{\bm{a}}[t-1] / \partial \bm{W}$ in the calculation of $\partial \widehat{\bm{a}}[t] / \partial \bm{W}$ will result in excessive influence of the gradient at $t-1$ on the parameter update at $t$.

By repeating the process from \eqref{A.2-2} to \eqref{A.2-3} in the same way, we derive
\begin{align}
    \frac{\partial \widehat{\bm{a}}^N[t]}{\partial \bm{W}^l}
=\left(\prod_{i=N-1}^{l+1}\frac{\partial \widehat{\bm{a}}^{i+1}[t]}{\partial \widehat{\bm{a}}^i[t]}\right)\frac{\partial \widehat{\bm{a}}^{l+1}[t]}{\partial \bm{W}^l}.
\end{align}
Therefore, the gradient of SAF-E is calculated as follows:
\begin{align}
\frac{\partial L_E[t]}{\partial \bm{W}^l}&=\frac{\partial L_E[t]}{\partial \widehat{\bm{a}}^N[t]}\left(\prod_{i=N-1}^{l+1}\frac{\partial \widehat{\bm{a}}^{i+1}[t]}{\partial \widehat{\bm{a}}^i[t]}\right)\frac{\partial \widehat{\bm{a}}^{l+1}[t]}{\partial \bm{W}^l}\\
&=\widehat{\bm{a}}^l[t]\,\frac{\partial L_E[t]}{\partial \widehat{\bm{a}}^N[t]}\left(\prod_{i=N-1}^{l+1}\frac{\partial \widehat{\bm{a}}^{i+1}[t]}{\partial \widehat{\bm{a}}^{i}[t]}\right)\frac{\partial \widehat{\bm{a}}^{l+1}[t]}{\partial \widehat{\bm{U}}^{l+1}[t]},\label{gr-SAF-E2}
\end{align}
where note that the Heaviside step function $H$ is element-wise. This concludes the derivation of~\eqref{gr-SAF-E} by setting
\begin{align}
\bm{g}_{\widehat{\bm{U}}}^{l+1}[t]
&=\frac{\partial L_E[t]}{\partial \widehat{\bm{a}}^N[t]}\left(\prod_{i=N-1}^{l+1}\frac{\partial \widehat{\bm{a}}^{i+1}[t]}{\partial \widehat{\bm{a}}^{i}[t]}\right)\frac{\partial \widehat{\bm{a}}^{l+1}[t]}{\partial \widehat{\bm{U}}^{l+1}[t]}.
\end{align}

Derivation of \eqref{gr-SAF-F} is almost the same as that of \eqref{gr-SAF-E}. As with SAF-E, suppose that $L_F$ depends only on $\widehat{\bm{a}}^l[T]$ and $\widehat{\bm{U}}^l[T]$. Since $L_F=\mathcal{L}(\widehat{\bm{a}}^N[t] / \sum_{t =0}^{T}\lambda^{T-t},\bm{y})$, \eqref{A.2-1} holds when $t$ is replaced with $T$, and $L_E[t]$ is replaced with $L_F$. The rest derivation is the same as in case \eqref{gr-SAF-E}. Then we have
\begin{align}
\frac{\partial L_F}{\partial \bm{W}^l}&=\widehat{\bm{a}}^l[T]\,\frac{\partial L_F}{\partial \widehat{\bm{a}}^N[T]}\left(\prod_{i=N-1}^{l+1}\frac{\partial \widehat{\bm{a}}^{i+1}[T]}{\partial \widehat{\bm{a}}^{i}[T]}\right)\frac{\partial \widehat{\bm{a}}^{l+1}[T]}{\partial \widehat{\bm{U}}^{l+1}[T]}
=\widehat{\bm{a}}^l[T]\,\bm{g}_{\widehat{\bm{U}}}^{l+1}[T].\label{gr-SAF-F2}
\end{align}
Note that $\partial \widehat{\bm{a}}^{l+1}[t] / \partial \widehat{\bm{U}}^{l+1}[t]$ in \eqref{gr-SAF-E2} and \eqref{gr-SAF-F2} is non-differentiable; we approximate it with the SG (refer to \eqref{SG}).

\setcounter{thm}{0}
\subsection{Proof of Theorem \ref{thm1}} \label{subsec_thm1}
\begin{thm}
The backward processes of SAF-E and OTTT$_{\rm O}$ are identical, that is, $\dfrac{\partial L_E[t]}{\partial \bm{W}^l}=\left(\dfrac{\partial L_E[t]}{\partial \bm{W}^l}\right)_{\rm OT}$.
\end{thm}
\begin{proof}
The gradient of SAF-E and OTTT$_{\rm O}$ are as follows:
\begin{align}
\frac{\partial L_E[t]}{\partial \bm{W}^l}&=\widehat{\bm{a}}^l[t]\,\bm{g}_{\widehat{\bm{U}}}^{l+1}[t], \quad
\left(\frac{\pd L_E[t]}{\pd \bm{W}^l} \right)_{\rm OT}= \widehat{\bm{a}}^l[t]\, \frac{\pd L_E[t]}{\pd \bm{s}^{N}[t]}\left(\prod_{i=N-1}^{l+1}\frac{\pd \bm{s}^{i+1}[t]}{\partial \bm{s}^{i}[t]}\right)\frac{\pd \bm{s}^{l+1}[t]}{\pd \bm{u}^{l+1}[t]}. 
\end{align}
We show that the gradient of SAF-E is equal to the gradient of OTTT$_{\rm O}$ by transforming $\bm{g}_{\widehat{\bm{U}}}^{l+1}[t]$. First, we calculate $\partial L_E[t]/\partial \widehat{\bm{a}}^N[t]$. Because $L_E[t]$ does not include any argument up to $t-1$, it holds that
\begin{align}
\frac{\partial L_E[t]}{\partial \bm{a}^N[t]}&=\frac{\partial L_E[t]}{\partial \bm{s}^N[t]}\frac{\partial \bm{s}^N[t]}{\partial \widehat{\bm{a}}^N[t]}=\frac{\partial L_E[t]}{\partial \bm{s}^N[t]}.\label{A.3-1}
\end{align}
Then, we have
\begin{align}
\bm{g}_{\widehat{\bm{U}}}^{l+1}[t]
&=
\frac{\partial L_E[t]}{\partial \bm{s}^N[t]}\left(\prod_{i=N-1}^{l+1}\frac{\partial \widehat{\bm{a}}^{i+1}[t]}{\partial \widehat{\bm{a}}^{i}[t]}\right)\frac{\partial \widehat{\bm{a}}^{l+1}[t]}{\partial \widehat{\bm{U}}^{l+1}[t]}.
\end{align}
Second, because of the forward process of SAF, i.e., \eqref{SAF-Def} and \eqref{eq:SAF-LIF}, we obtain that
\begin{align}
\frac{\partial \widehat{\bm{a}}^{l+1}[t]}{\partial \widehat{\bm{U}}^{l+1}[t]}&=\delta(\widehat{\bm{U}}^{l+1}[t] - V_{\rm th}(\lambda\widehat{\bm{a}}^{l+1}[t-1] + 1)\\
&=\delta(\bm{u}^{l+1}[t]-V_{\rm th}),
\end{align}
where $\delta$ represents the delta function.
On the other hand, the following equation can be derived from the forward process of OTTT$_{\rm O}$, i.e., \eqref{define-snn}:
\begin{align}
\frac{\partial \bm{s}^{l+1}[t]}{\partial \bm{u}^{l+1}[t]}&=\delta(\bm{u}^{l+1}[t] - V_{\rm th}).  
\end{align}
Hence,
\begin{align}
\frac{\partial \widehat{\bm{a}}^{l+1}[t]}{\partial \widehat{\bm{U}}^{l+1}[t]}=\frac{\partial \bm{s}^{l+1}[t]}{\partial \bm{u}^{l+1}[t]}. \label{A.3-2}
\end{align}
Then, we have that
\begin{align}
\bm{g}_{\widehat{\bm{U}}}^{l+1}[t]&=
\frac{\partial L_E[t]}{\partial \bm{s}^N[t]}\left(\prod_{i=N-1}^{l+1}\frac{\partial \widehat{\bm{a}}^{i+1}[t]}{\partial \widehat{\bm{a}}^{i}[t]}\right)\frac{\partial \bm{s}^{l+1}[t]}{\partial \bm{u}^{l+1}[t]}.
\end{align}
From \eqref{SAF-Def}, $\widehat{\bm{a}}^{i+1}[t]$ depends on $\widehat{\bm{a}}^{i+1}[t-1]$ and $\widehat{\bm{U}}^{i}[t]$, while $\widehat{\bm{a}}^{i+1}[t-1]$ does not depend on $\widehat{\bm{a}}^{i+1}[t]$ so $\partial \widehat{\boldsymbol{a}}^{i+1}[t-1] / \partial \widehat{\boldsymbol{a}}^i[t]=0$. Also $\partial \widehat{\bm{U}}^{i+1}[t] / \partial \widehat{\bm{a}}^i[t]=\bm{W}^i$ since the forward process of SAF \eqref{SAF-Def}. Thus, we calculate that
\begin{align} \label{da/da}
\frac{\partial \widehat{\bm{a}}^{i+1}[t]}{\partial \widehat{\bm{a}}^{i}[t]}=\frac{\partial \widehat{\bm{a}}^{i+1}[t]}{\partial \widehat{\bm{U}}^{i+1}[t]}\frac{\partial \widehat{\bm{U}}^{i+1}[t]}{\partial \widehat{\bm{a}}^i[t]}+\frac{\partial \widehat{\bm{a}}^{i+1}[t]}{\partial \widehat{\bm{a}}^{i+1}[t-1]}\frac{\partial \widehat{\bm{a}}^{i+1}[t-1]}{\partial \widehat{\bm{a}}^i[t]}
=\frac{\partial \widehat{\bm{a}}^{i+1}[t]}{\partial \widehat{\bm{U}}^{i+1}[t]}\bm{W}^i.
\end{align}
Approximating $\partial \widehat{\boldsymbol{a}}^{i+1}[t] / \partial \widehat{\boldsymbol{U}}^{i+1}[t]$ with the SG (refer to \eqref{SG}),
$\partial \widehat{\boldsymbol{a}}^{i+1}[t] / \partial \widehat{\boldsymbol{a}}^{i}[t]$ becomes implementable. Moreover, $\bm{s}^{i+1}[t]$ depends on $\bm{u}^{i+1}[t]$ because the forward process of OTTT$_{\rm O}$ is given by~\eqref{define-snn}. Hence, 
\begin{align}
\frac{\partial \bm{s}^{i+1}[t]}{\partial \bm{s}^{i}[t]}&=\frac{\partial \bm{s}^{i+1}[t]}{\partial \bm{u}^{i+1}[t]}\frac{\partial \bm{u}^{i+1}[t]}{\partial \bm{s}^{i}[t]}=\frac{\partial \bm{s}^{i+1}[t]}{\partial \bm{u}^{i+1}[t]}\bm{W}^i .\label{A.3-3}
\end{align}
Combining~\eqref{A.3-2}, \eqref{da/da} and~\eqref{A.3-3}, we have 
\begin{align}
\dfrac{\partial \widehat{\bm{a}}^{i+1}[t]}{\partial \widehat{\bm{a}}^{i}[t]}=\dfrac{\partial \bm{s}^{i+1}[t]}{\partial \bm{s}^{i}[t]}.
\end{align}
In the end, we transform $\bm{g}_{\widehat{\bm{U}}}^{l+1}[t]$ as follows:
\begin{align}
\bm{g}_{\widehat{\bm{U}}}^{l+1}[t]&=
\frac{\partial L_E[t]}{\partial \bm{s}^N[t]}\left(\prod_{i=N-1}^{l+1}\frac{\partial \bm{s}^{i+1}[t]}{\partial \bm{s}^{i}[t]}\right)\frac{\partial \bm{s}^{l+1}[t]}{\partial \bm{u}^{l+1}[t]}.\label{A.3-4}
\end{align}
This concludes the proof. 
\end{proof}
\subsection{Proof of Theorem \ref{thm2}}\label{Proof of Theorem2}
\begin{thm}
Suppose that $\bm{m}[t]$ converges when $t\rightarrow\infty$. Then, for sufficiently large $T$, the backward processes of SAF-F and Spike Representation are identical up to a scale factor, that is,
$\dfrac{\partial L_F}{\partial \bm{W}^l}=V_{\rm th}\left(\dfrac{\partial L_F}{\partial \bm{W}^l}\right)_{\rm SR}$.
\end{thm}
\begin{proof}

From assumptions, $\bm{a}^{l+1}[T]\thickapprox \sigma\left((\bm{W}^l\bm{a}^l[T]+\bm{b}^{l+1}) / V_{\rm th} \right)$, where $\sigma(x)=\min(\max(0,x),1)$. Therefore, the followings hold for $i=l+1,\ldots,N-1$: 
\begin{align}
\frac{\partial L_F}{\partial \bm{W}^l}&=\frac{\partial L_F}{\partial \bm{a}^N[T]}\frac{\partial \bm{a}^N[T]}{\partial \bm{W}^l},\label{A.4-1}\\
\frac{\partial \bm{a}^{i+1}[T]}{\partial \bm{W}^l}&=\frac{\partial \bm{a}^{i+1}[T]}{\partial \bm{a}^i[T]}\frac{\partial \bm{a}^i[T]}{\partial \bm{W}^l}.\label{A.4-2}
\end{align}
By repeatedly substituting \eqref{A.4-2} for \eqref{A.4-1}, we can calculate $\pd L_F / \pd \bm{W}^l$ as follows:
\begin{align}
\frac{\pd L_F}{\pd \bm{W}^l} = \frac{\pd L_F}{\pd \bm{a}^{N}[T]}\left(\prod_{i=N-1}^{l+1}\frac{\pd \bm{a}^{i+1}[T]}{\partial \bm{a}^{i}[T]}\right)\frac{\pd\bm{a}^{l+1}[T]}{\pd\bm{W}^l}.\label{A.4-3}
\end{align}

This is the gradient of Spike Representation and denote it by $(\pd L_F / \pd \bm{W}^l)_{\rm SR}$. We will show that $(\pd L_F / \pd \bm{W}^l)_{\rm SR}$ is in proportion to ${\partial L_F}/{\partial \bm{W}^l}=\widehat{\bm{a}}^l[T]\,\bm{g}_{\widehat{\bm{U}}}^{l+1}[T]$, which is the gradient of SAF-F.

First, let $\Lambda=\sum_{\tau=0}^{T}\lambda^{T-\tau}$. Then $\bm{a}^l[T]=\widehat{\bm{a}}^l[T] / \Lambda$ for any layer $l$. From linearity of differentiation and change of variables, we have followings:
\begin{align}\label{3derivative}
    \frac{\partial L_F}{\partial \bm{a}^N[T]}=\frac{\partial L_F}{\Lambda\partial \widehat{\bm{a}}^N[T]},
    \quad
    \frac{\partial \bm{a}^{i+1}[T]}{\partial \bm{a}^{i}[T]}=\frac{\partial \widehat{\bm{a}}^{i+1}[T]}{\partial \widehat{\bm{a}}^{i}[T]},
    \quad
    \frac{\partial \bm{a}^{l+1}[T]}{\partial \bm{W}^l}=\frac{\Lambda\partial \widehat{\bm{a}}^{l+1}[T]}{\partial \bm{W}^l}.
\end{align}
Substituting~\eqref{3derivative} into \eqref{A.4-3}, then we obtain
\begin{align}
\left(\frac{\partial L_F}{\partial \bm{W}^l}\right)_{\rm SR}&=\frac{\partial L_F}{\partial \widehat{\bm{a}}^N[T]}\left(\prod_{i=N-1}^{l+1}\frac{\partial \widehat{\bm{a}}^{i+1}[T]}{\partial \widehat{\bm{a}}^{i}[T]}\right)\frac{\partial \widehat{\bm{a}}^{l+1}[T]}{\partial \bm{W}^l}.
\end{align}
Second, it follows from $\bm{a}^l[T]=\widehat{\bm{a}}^l[T] / \Lambda$ and $\bm{a}^{l+1}[T]\thickapprox \sigma\big((\bm{W}^l\bm{a}^l[T]+\bm{b}^{l+1}) / V_{\rm th}\big)$ that
\begin{align}
\widehat{\bm{a}}^{l+1}[T]&\thickapprox\Lambda\sigma\left(\frac{1}{V_{\rm th}}\left(\frac{1}{\Lambda}\bm{W}^l\widehat{\bm{a}}^l[T]+\bm{b}^{l+1}\right)\right).
\end{align}
Here, taking care that $\sigma$ is element-wise, we calculate as follows:
\begin{align}
\left(\frac{\partial L_F}{\partial \bm{W}^l}\right)_{\rm SR}&=\frac{\widehat{\bm{a}}^l[T]}{V_{\rm th}\Lambda}\,\left(\frac{\partial L_F}{\partial \widehat{\bm{a}}^N[T]}\left(\prod_{i=N-1}^{l+1}\frac{\partial \widehat{\bm{a}}^{i+1}[T]}{\partial \widehat{\bm{a}}^{i}[T]}\right)\odot\Lambda \sigma '\left(\frac{1}{V_{\rm th}}\left(\frac{1}{\Lambda}\bm{W}^l\widehat{\bm{a}}^l[T]+\bm{b}^{l+1}\right)^{\top}\right)\right)\\
&=\frac{\widehat{\bm{a}}^l[T]}{V_{\rm th}}\left(\frac{\partial L_F}{\partial \widehat{\bm{a}}^N[T]}\left(\prod_{i=N-1}^{l+1}\frac{\partial \widehat{\bm{a}}^{i+1}[T]}{\partial \widehat{\bm{a}}^{i}[T]}\right)\odot \bm{d}^{l+1}[T]^{\top}\right),
\end{align}
where we set $\bm{d}^{l+1}[T]=\sigma'\left((\bm{W}^l\widehat{\bm{a}}^l[T] / \Lambda +\bm{b}^{l+1}) / V_{\rm th} \right)$,
 and $\odot$ is the element-wise product.
Now we assume that
\begin{align}
\frac{\partial \bm{s}^{l+1}[T]}{\partial \bm{u}^{l+1}[T]}=\operatorname{diag}(\bm{d}^{l+1}[T]),
\end{align}
for any $l=0,\ldots,N-1$, where $\operatorname{diag}(\bm{d}^{l+1}[T])$ is a diagonal matrix constructed from $\bm{d}^{l+1}[T]$. The reason why this assumption is valid discussed in~\citet{Xiao2022OnlineNetworks}. Then, we obtain from \eqref{A.3-2} that

\begin{align}
\left(\frac{\partial L_F}{\partial \bm{W}^l}\right)_{\rm SR}&=\frac{\widehat{\bm{a}}^l[T]}{V_{\rm th}}\,\frac{\partial L_F}{\partial \widehat{\bm{a}}^N[T]}\left(\prod_{i=N-1}^{l+1}\frac{\partial \widehat{\bm{a}}^{i+1}[T]}{\partial \widehat{\bm{a}}^{i}[T]}\right)\frac{\partial \widehat{\bm{a}}^{l+1}[T]}{\partial \widehat{\bm{U}}^{l+1}[T]}
=\frac{1}{V_{\rm th}}\widehat{\bm{a}}^l[T]\,\bm{g}_{\widehat{\bm{U}}}^{l+1}[T].
\end{align}
\end{proof}

\subsection{Proof of Corollary \ref{corollary3}}\label{subsec:Proof of Corollary3}
\begin{cor}
For SNN with a feedforward connection \eqref{define-snn-F}, or a feedback connection \eqref{define-snn-B}, the following hold.
\item $\mathrm{(i)}$ The backward processes of SAF-E and OTTT$_{\rm O}$ with a feedforward connection are identical.
\item $\mathrm{(ii)}$ Suppose that $\bm{m}[t]$ converges when $t\rightarrow\infty$. Then, for sufficiently large $T$, the backward processes of SAF-F and Spike Representation with a feedforward connection are identical up to a scale factor.
\item $\mathrm{(iii)}$ The backward processes of SAF-E and OTTT$_{\rm O}$ with a feedback connection are identical. 
\end{cor}
\begin{proof}
First, we show $\mathrm{(i)}$. The gradients for parameters other than $\bm{W}_f$ are equal to the gradients when there is no feedforward connection. 
From Theorem~\ref{thm1}, these gradients are same as OTTT$_{\rm O}$. 
Therefore, we only need to check $\partial L_E[t]/\partial \bm{W}_f=(\partial L_E[t]/\partial \bm{W}_f)_{\rm OT}$. 
In fact, the gradient of OTTT$_{\rm O}$ calculated (see \citet{Xiao2022OnlineNetworks}) as
\begin{align}
\left(\frac{\pd L_E[t]}{\pd \bm{W}_f} \right)_{\rm OT}=\widehat{\bm{a}}^{p}[t]\, \frac{\pd L[t]}{\pd \bm{s}^{N}[t]}\left(\prod_{i=N-1}^{q+1}\frac{\pd \bm{s}^{i+1}[t]}{\partial \bm{s}^{i}[t]}\right)\frac{\pd \bm{s}^{q+1}[t]}{\pd \bm{u}^{q+1}[t]},
\end{align}
and \eqref{A.3-4} holds, then we have
\begin{align}
\left(\frac{\pd L_E[t]}{\pd \bm{W}_f} \right)_{\rm OT}=\widehat{\bm{a}}^{p}[t]\,\bm{g}_{\widehat{\bm{U}}}^{q+1}[t]=\frac{\pd L_E[t]}{\pd \bm{W}_f} .
\end{align}
Hence, SAF-E is equivalent to OTTT$_{\rm O}$ even if there is a feedforward connection in SNN. The same method can be used to prove $\mathrm{(iii)}$.

Next, we show $\mathrm{(ii)}$. The gradients for parameters other than $\bm{W}_f$ remain the same as when there is no feedback connection. 
According to Theorem~\ref{thm2}, these gradients are identical up to a scale factor to Spike Representation. Therefore, we only need to check $\partial L_F[t]/\partial \bm{W}_b=V_{\rm th}(\partial L_F[t]/\partial \bm{W}_f)_{\rm SR}$. When there is a feedforward connection, \eqref{grad-SR} holds. Therefore, it can be proved in the same way as Sec.~\ref{Proof of Theorem2}, noting that $\bm{a}^{q+1}[T]\thickapprox \sigma\left((\bm{W}^q\bm{a}^q[T]+\bm{b}^{l+1}+\bm{W}_f\bm{a}^p[T] / V_{\rm th} \right)$.
\end{proof}

\subsection{Proof of Theorem \ref{thm3}}\label{subsec:Proof of Theorem3}

\begin{thm}
Suppose that $\bm{m}[t]$ converges when $t\rightarrow\infty$. Then, for sufficiently large $T$, the backward processes of SAF-F and Spike Representation with a feedback connection are similar, that is, $\left\langle \dfrac{\partial L_F}{\partial \bm{\theta}}, \left(\dfrac{\partial L_F}{\partial \bm{\theta}}\right)_{\rm SR}\right\rangle>0$ for all parameters $\bm{\theta}$.
\end{thm}
\begin{proof}From the assumption, the firing rates of each layer converge to these equilibrium points: $(\bm{a}^{l+1})^*=f_{l+1}((\bm{a}^l)^*)$~$(l \neq  q)$, $(\bm{a}^{q+1})^*=f_{q+1}(f_p\circ\cdots\circ f_{q+2}((\bm{a}^{q+1})^*),(\bm{a}^q)^*)=f_{q+1}((\bm{a}^p)^*,(\bm{a}^q)^*)$, where $f_{l+1}((\bm{a}^l)^*)=\sigma\left((\bm{W}^l(\bm{a}^l)^*+\bm{b}^{l+1})/V_{\rm th}\right)$, and $f_{q+1}((\bm{a}^p)^*,(\bm{a}^q)^*)=\sigma\left((\bm{W}^q(\bm{a}^q)^*+\bm{b}^{q+1}+\bm{W}_b(\bm{a}^p)^*/V_{\rm th}\right)$.
Since $T \gg 1$, $\pd L_F / \pd \bm{\theta}$ can be calculated using the spike representation as follows (see \citet{Xiao2021TrainingState},\citet{Xiao2022OnlineNetworks}):

\begin{align}
\left(\frac{\pd L_F}{\pd \bm{\theta}}\right)_{\rm SR}= \frac{\pd L_F}{\pd \bm{a}^{l+1}[T]}\left(I-\frac{\pd f_{l+1}}{\pd \bm{a}^l[T]}\right)^{-1}\frac{\pd f_{l+1}(\bm{a}^l[T])}{\pd\bm{\theta}}, \label{B.2-1}
\end{align}
where 
$\pd f_{l+1}/\pd \bm{a}^l[T]$ denotes Jacobian matrix and $I$ denotes identity matrix. 
Here, we regard  
$(I-\pd f_{l+1}/\pd \bm{a}^l[T])^{-1}$ of \eqref{B.2-1} as an identity matrix, and denote it by
\begin{align}
\widetilde{\left(\frac{\pd L_F}{\pd \bm{\theta}} \right)}_{\rm SR}&= \frac{\pd L_F}{\pd \bm{a}^{l+1}[T]}\frac{\pd f_{l+1}(\bm{a}^l[T])}{\pd\bm{\theta}}.
\label{B.2-2}\end{align}

Then, it is proved in~\citet{jfb2022} and~\citet{OntrainingImplicitModels} that 
\begin{align}
\left\langle\widetilde{\left(\dfrac{\pd L_F}{\pd \bm{\theta}}\right)}_{\rm SR}, \left(\dfrac{\pd L_F}{\pd \bm{\theta}} \right)_{SR}\right\rangle > 0,
\end{align}
where $\langle \cdot , \cdot \rangle$ denotes inner product. 
If $\bm{\theta}$ is not $\bm{W}_b$, right-hand side of \eqref{B.2-2} is equal to right-hand side of \eqref{A.4-3}, then from Theorem~\ref{thm2}, 
\begin{align}
\widetilde{\left(\dfrac{\pd L_F}{\pd \bm{\theta}}\right)}_{\rm SR} = \frac{1}{V_{\rm th}} \frac{\pd L_F}{\pd \bm{\theta}}. 
\end{align}
Next, we consider $\bm{\theta}$ as $ \bm{W}_b$. $\bm{a}^{q+1}[T]$ can be approximated by  $\sigma\left((\bm{W}^q\bm{a}^q[T]+\bm{b}^{q+1}+\bm{W}_b\,\bm{a}^p[T])/V_{\rm th}\right)$ and $\bm{a}^p[T]\approx\bm{a}^p[T-1]$ because $T$ is large. Therefore, $\bm{a}^{q+1}[T] \approx \sigma\left((\bm{W}^q\bm{a}^q[T]+\bm{b}^{q+1}+\bm{W}_b\,\bm{a}^p[T-1])/V_{\rm th}\right)$. 
Setting $\bm{d}_b^{l+1}[T]=\sigma'\left((\bm{W}^l\bm{a}^l[T]+\bm{b}^{l+1}+\bm{W}_b\,\bm{a}^k[T-1])/V_{\rm th}\right)$ and calculating similar to Sec.~\ref{Proof of Theorem2}, we derive that
\begin{align}
\widetilde{\left(\frac{\pd L_F}{\pd \bm{W}_b}\right)}_{\rm SR}&=\frac{\widehat{\bm{a}}^k[T-1]}{V_{\rm th}}\,\left(\frac{\partial L_F}{\partial \widehat{\bm{a}}^N[T]}\left(\prod_{i=N-1}^{l+1}\frac{\partial \widehat{\bm{a}}^{i+1}[T]}{\partial \widehat{\bm{a}}^{i}[T]}\right)\odot\bm{d}_b^{l+1}[T]\right)\\
&=\frac{1}{V_{\rm th}}\widehat{\bm{a}}^p[T-1]\,\bm{g}_{\widehat{\bm{U}}}^{q+1}[T]\\
&=\frac{1}{V_{\rm th}}\left(\frac{\pd L_F}{\pd \bm{W}_b}\right).
\end{align}
In the end, take the inner product between their gradient  we obtain that
\begin{align}
\left\langle\left(\frac{\pd L_F}{\pd \bm{\theta}}\right),\left(\frac{\pd L_F}{\pd \bm{\theta}} \right)_{SR}\right\rangle=
V_{\rm th}\left\langle\widetilde{\left(\frac{\pd L_F}{\pd \bm{\theta}}\right)}_{\rm SR},\left(\frac{\pd L_F}{\pd \bm{\theta}} \right)_{SR}\right\rangle >0.
\end{align}
\end{proof}

\section{Implementation detail}
\label{sec:implementation}
In the experiments of our study, we used the VGG network as follows:

64C3-128C3-AP2-256C3-256C3-AP2-512C3-512C3-AP2-512C3-512C3-GAP-FC,

where $x$C$y$ represents the convolutional layer with $x$-output channels and $y$-stride, AP$x$ represents the average pooling with the kernel size 2, GAP represents the global average pooling, and FC represents the fully connected layer. 

As for the time step, we set it to 6 if there is no mention. We used the stochastic gradient descent~(SGD) as the optimizer with the batch size, epoch, initial learning rate for the cosine annealing, and momentum at 128, 300, 0.1, and 0.9. As for the loss function, we applied the combination of the cross-entropy~(CE) and the mean-squared error~(MSE) losses, i.e., $L = (1-\alpha){\rm CE} + \alpha{\rm MSE}$~(where $\alpha = 0.05$). In addition, we set the leaky term as $\lambda = 0.5$ and the threshold as $V_{\rm th}=1$ and used the scaled weight standardization (sWS) \citep{qiao2019micro}. We applied derivative of a sigmoid function as a surrogate gradient:
\begin{align} \label{SG}
\frac{\partial \widehat{a}^{l+1}[t]}{\partial \widehat{U}^{l+1}[t]} := \frac{1}{\beta} \frac{\exp((V_{\rm th}-\widehat{U}^{l+1}[t])/\beta)}{(1+\exp((V_{\rm th}-\widehat{U}^{l+1}[t])/\beta))^2},
\end{align}
where $\beta$ is a hyperparameter and we set $\beta=4$. 
 Note that all settings, including the above, were the same as \citet{Xiao2022OnlineNetworks}.

\section{Comparison of gradient}\label{Comparison_of_gradient}
Tables~\ref{tab:my_label} and \ref{tab:my_label2} show the correlation coefficients and mean absolute errors~(MAE) of gradients of the SAF and OTTT. These values were computed in the input layer after training, where the gradient error between SAF and OTTT were the largest. As can be seen from the table, it is considered that SAF-E and OTTT$_{\rm O}$ have the same gradients, while SAF-F and OTTT$_{\rm A}$ have different (but close direction) gradients. These results clearly support our theory.

\section{Experimental Analysis of SAF-E with Feedback and FeedForward Connection} \label{appendix_experiment}
In this section, we confirm that SAF-E and OTTT$_{\rm O}$ are equivalent even if there is a feedback or feedforward connection. First, we compare the performance of SAF-E with OTTT$_{\rm O}$ when both networks have a feedback connection from the first layer to the $N$-th layer.  The setup was the same as in the experiments of \citet{Xiao2022OnlineNetworks}.

Table~\ref{table/each_time_feedback} shows the accuracy and total firing rate when we set $T=6$. As shown in this table, SAF-E and OTTT$_{\rm O}$ with feedback connection are almost close. Also, the change due to inference with SNNs consisting of LIF neurons is also almost negligible. Figure~\ref{fig/each_time_feedback} shows the accuracy, loss curve, and the firing rates of each layer. From these results, we confirmed experimentally that SAF-E and OTTT$_{\rm O}$ with feedback connection are close, which consistent with the assertion in Corollary.~\ref{corollary3}.

Next, we compare the performance of SAF-E with OTTT$_{\rm O}$ when both networks have a feedforward connection. We used the RepVGG network \citep{Xiaohan2021repvgg}. The results are shown in Table~\ref{table/each_time_feedforward} and Fig.~\ref{fig/each_time_feedforward}. As in the case of feedback connections, we can also see that SAF-E and OTTT$_{\rm O}$  with a feedforward connection are close. 

\begin{table}[t]
 \caption{Gradient comparison of SAF-E and OTTT$_{\rm O}$ on CIFAR-10 and CIFAR-100.} 
    \centering
    \begin{tabular}{lccc}
    \hline
     \bf{Dataset}&\bf{$T$}&\bf{Correlation Coefficient}&\bf{MAE}~($\times 10^{-6}$)  \\
         \hline
        CIFAR-10&6&0.984$\pm$0.001&2.26$\pm$0.10\\
        CIFAR-100&6&0.969$\pm$0.001&12.4$\pm$0.40\\
        \hline
    \end{tabular}
    \label{tab:my_label}
\end{table}
\begin{table}[t]
     \caption{Gradient comparison of SAF-F and OTTT$_{\rm A}$ on CIFAR-10 and CIFAR-100.} 
    \centering
    \begin{tabular}{lccc}
    \hline         \bf{Dataset}&\bf{$T$}&\bf{Correlation Coefficient}&\bf{MAE}~($\times 10^{-6}$)  \\
         \hline
        CIFAR-10&6&0.566$\pm$0.097&23.2$\pm$4.40\\
        CIFAR-10&32&0.557$\pm$0.064&22.5$\pm$4.00\\
        CIFAR-100&6&0.590$\pm$0.021&72.6$\pm$12.8\\
        CIFAR-100&32&0.557$\pm$0.030&73.6$\pm$17.2\\
        \hline
    \end{tabular}
    \label{tab:my_label2}
\end{table}

\begin{table}[t]
  \caption{Comparison of SAF-E and OTTT$_{\rm O}$ with feedback connection on CIFAR-10~($T=6$). The values in parentheses are the changes in firing rate and accuracy due to inference by the SNN composed of LIF. }
  \centering
  \begin{tabular}{lcc}   \bf{Method} &\bf{Firing rate~[\%]}& \bf{Accuracy~[\%]}  \\
    \hline
    OTTT$_{\rm O}$ with feedback & 14.74$\pm$0.34&  93.23$\pm$0.28\\
    SAF-E with feedback~(ours) & 14.32$\pm$0.12 ($1.78 \times 10^{-6}$)& 93.20$\pm$0.18 ($6.67 \times 10^{-5}$)\\
    \hline  \end{tabular}\label{table/each_time_feedback}
\end{table}

\begin{table}[t]
  \caption{Comparison of SAF-E and OTTT$_{\rm O}$ with feedforward connection on CIFAR-10~($T=6$). The values in parentheses are the changes in firing rate and accuracy due to inference by the SNN composed of LIF. }
  \centering
  \begin{tabular}{lcc}   \bf{Method} &\bf{Firing rate~[\%]}& \bf{Accuracy~[\%]}  \\
    \hline
    OTTT$_{\rm O}$ with feedforward & 15.03$\pm$0.10&  93.53$\pm$0.15\\
    SAF-E with feedforward~(ours) & 14.72$\pm$0.04 (3.342$\times10^{-5}$)& 93.53$\pm$0.17 (0.004)\\
    \hline  \end{tabular}\label{table/each_time_feedforward}
\end{table}

\begin{figure}[t]
  \begin{minipage}[b]{0.33\linewidth}
    \centering
\includegraphics[keepaspectratio, scale=0.36]{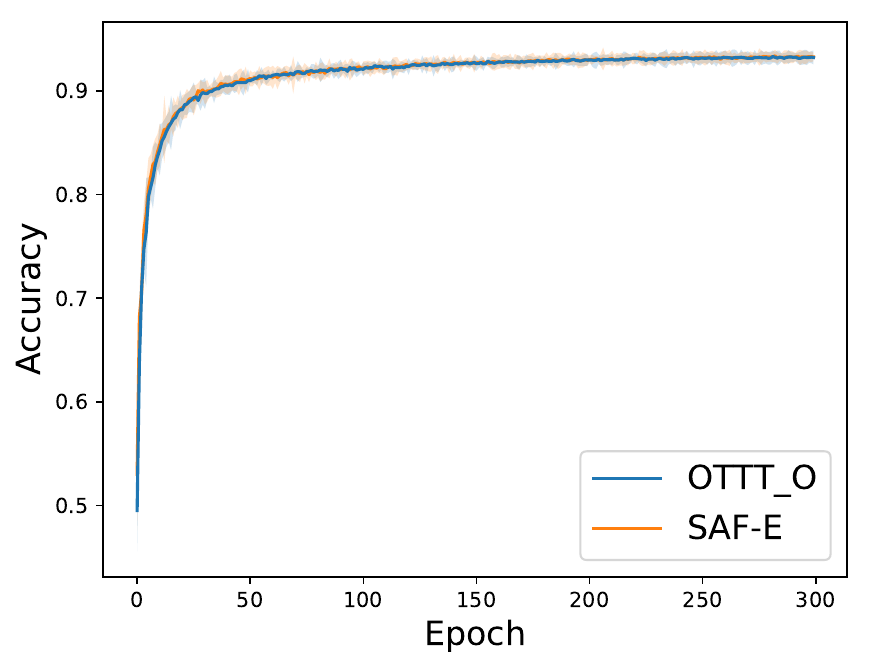} \label{acc_each_time_feedback}
  \end{minipage}
  \begin{minipage}[b]{0.33\linewidth}
    \centering
\includegraphics[keepaspectratio, scale=0.36]{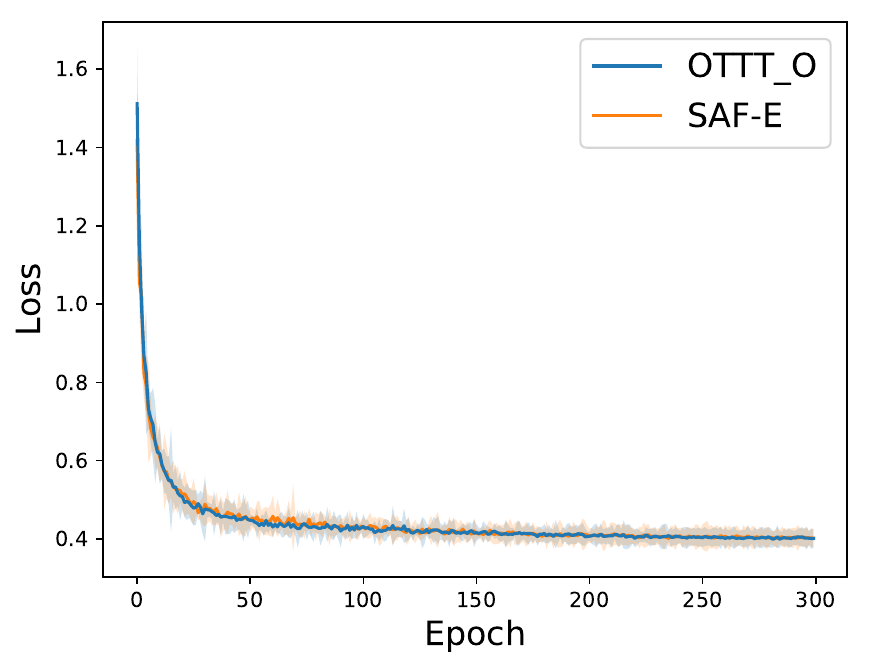}
\label{loss_each_time_feedback}
  \end{minipage}
  \begin{minipage}[b]{0.33\linewidth}
    \centering
\includegraphics[keepaspectratio, scale=0.36]{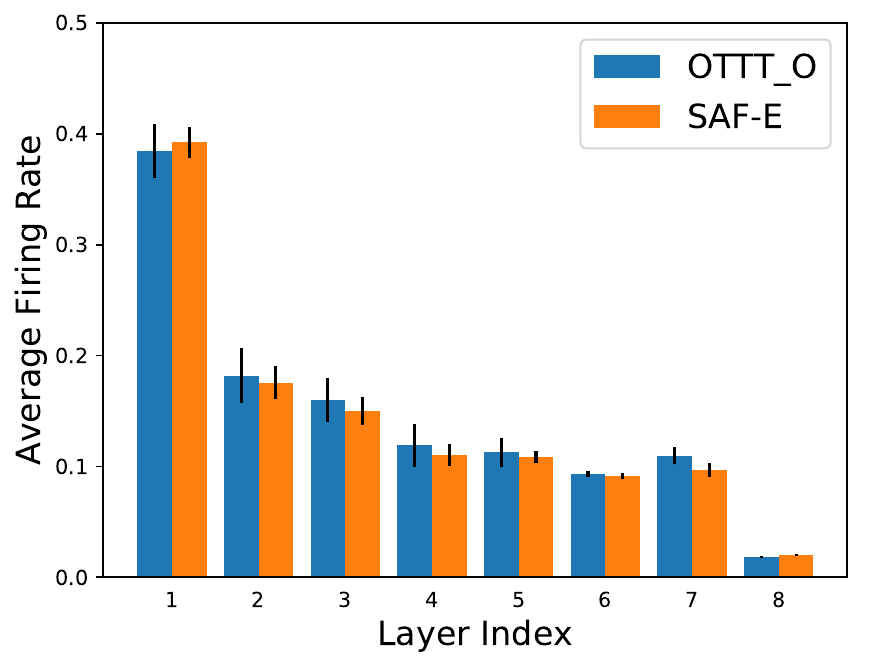} \label{firing_rate_each_time_feedback}
  \end{minipage}
  \caption{Accuracy and loss curves, and firing rates of each layer of SAF-E and OTTT$_{\rm O}$ with feedback connection on CIFAR-10~($T=6$).}\label{fig/each_time_feedback}
\end{figure}

\begin{figure}[t]
  \begin{minipage}[b]{0.33\linewidth}
    \centering
\includegraphics[keepaspectratio, scale=0.36]{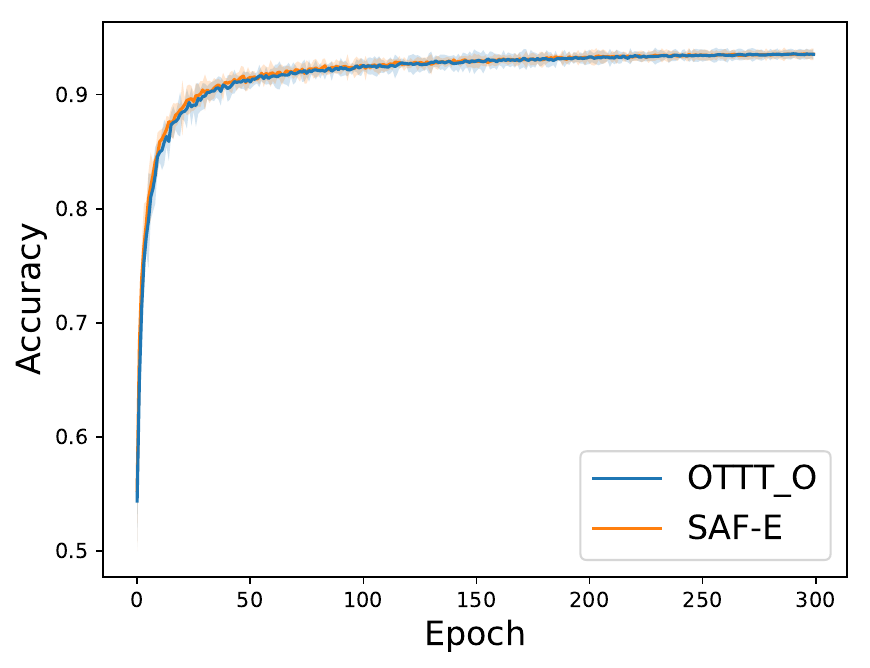} \label{acc_each_time_feedforward}
  \end{minipage}
  \begin{minipage}[b]{0.33\linewidth}
    \centering
\includegraphics[keepaspectratio, scale=0.36]{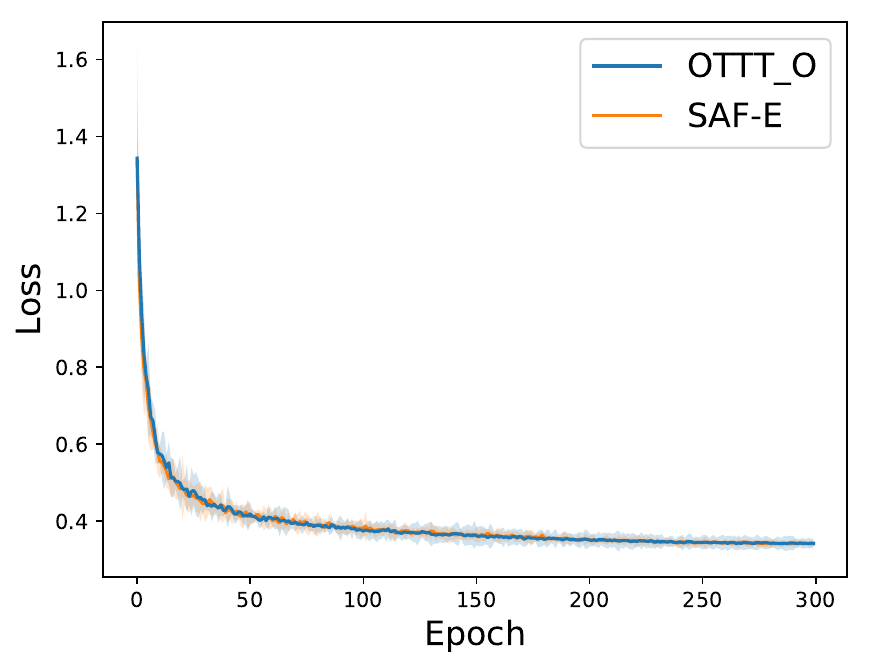}
\label{loss_each_time_feedforward}
  \end{minipage}
  \begin{minipage}[b]{0.33\linewidth}
    \centering
\includegraphics[keepaspectratio, scale=0.36]{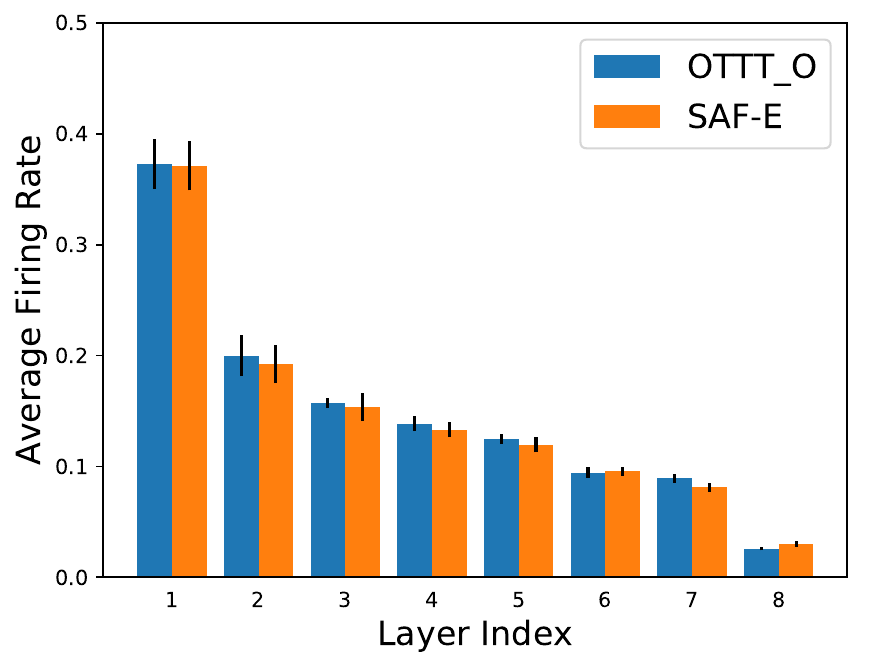} \label{firing_rate_each_time_feedforward}
  \end{minipage}
  \caption{Accuracy and loss curves, and firing rates of each layer of SAF-E and OTTT$_{\rm O}$ with feedforward connection on CIFAR-10~($T=6$).}\label{fig/each_time_feedforward}
\end{figure}
\end{document}